\documentclass{article}
\pdfoutput=1
\usepackage{arxiv}

\usepackage[utf8]{inputenc} % allow utf-8 input
\usepackage[T1]{fontenc}    % use 8-bit T1 fonts
\usepackage{hyperref}       % hyperlinks
\usepackage{url}            % simple URL typesetting
\usepackage{booktabs}       % professional-quality tables
\usepackage{amsfonts}       % blackboard math symbols
\usepackage{amsmath}       % blackboard math symbols
\usepackage{nicefrac}       % compact symbols for 1/2, etc.
\usepackage{microtype}      % microtypography
\usepackage{lipsum}		% Can be removed after putting your text content
\usepackage{graphicx}
\usepackage{doi}
\usepackage{amsthm}
\usepackage{subfigure}

\newtheorem{theorem}{Theorem}
\newtheorem{lemma}[theorem]{Lemma}
\newtheorem{proposition}[theorem]{Proposition}
\newtheorem{remark}[theorem]{Remark}
\newtheorem{corollary}[theorem]{Corollary}

\usepackage[numbers]{natbib}
\usepackage{url}
\usepackage{amssymb}

\usepackage{dsfont}
\usepackage{float}
\newcommand{\norm}[1]{\left\lVert#1\right\rVert}

\renewcommand{\Re}{\operatorname{Re}}
\renewcommand{\Im}{\operatorname{Im}}

\newcommand{\dd}[1]{\mathrm{d}#1}

\newcommand{\e}{\text {e}}

\title{Rank-one matrix estimation: analytic time evolution \\ of gradient descent dynamics}

\renewcommand{\shorttitle}{Rank-one matrix gradient descent dynamics}

\author{ 
  {\hspace{1mm}Antoine Bodin} \\
    Communication Theory Laboratory\\
    École Polytechnique Fédérale de Lausanne\\
    Switzerland\\
    \texttt{antoine.bodin@epfl.ch} \\
	\And
  {\hspace{1mm}Nicolas Macris} \\
  Communication Theory Laboratory\\
  École Polytechnique Fédérale de Lausanne\\
  Switzerland\\
  \texttt{nicolas.macris@epfl.ch} \\
}

\hypersetup{
pdftitle={Rank-one matrix gradient descent dynamics},
pdfauthor={Antoine Bodin, Nicolas Macris}
}

\begin{document}

\maketitle

\begin{abstract}%
We consider a rank-one symmetric matrix corrupted by additive noise. The rank-one matrix is formed by an $n$-component unknown vector on the sphere of radius $\sqrt{n}$, and we consider the problem of estimating this vector from the corrupted matrix in the high dimensional limit of $n$ large, by gradient descent for a quadratic cost function on the sphere.  Explicit formulas for the whole time evolution of the overlap between the estimator and unknown vector, as well as the cost, are rigorously derived. In the long time limit we recover the well known spectral phase transition, as a function of the signal-to-noise ratio. The explicit formulas also allow to point out interesting transient features of the time evolution. Our analysis technique is based on recent progress in random matrix theory and uses {\it local versions} of the semi-circle law.
\end{abstract}


\section{Introduction} 
Gradient descent dynamic is at the root of machine learning methods, and in particular, its stochastic version augmented by various ad-hoc methods, has been very successful at finding "good" minima of cost functions \cite{Lecun98gradient-basedlearning}. However, rigorous detailed results on the full time evolution of the dynamics are scarce even for simple models and usual gradient descent.
In this contribution, we show how to completely solve for the whole time evolution for a simple paradigm of non-linear estimation; the problem of estimating a rank-one spike embedded in noise. 

Let 
$\theta^* \in \mathbb{S}^{n-1}(\sqrt n)$ a {\it hidden} vector on the $n-1$ dimensional sphere of radius $\sqrt n$, i.e., $\theta^* = (\theta_1^*, \ldots, \theta_n^*)^T$ and $\Vert \theta^*\Vert_2^2 = n$. We consider the {\it data} matrix $Y$ with elements
$Y = \theta^* \theta^{*T} + \sqrt \frac{n}{\lambda} \xi$
where $\lambda > 0$ is the signal-to-noise parameter and $\xi=(\xi_{i,j})_{1 \leq i,j\leq n}$ a symmetric random noise matrix with i.i.d $\xi_{i,j}$ for $i\leq j$. The goal is to recover $\theta^*$ given that $Y$ and $\lambda$ are known.
This model is usually considered for a gaussian noise symmetric matrix $\xi_{ij}\sim \mathcal{N}(0,1)$, $i\leq j$, and is variously called the noisy rank-one matrix estimation problem or the spiked Wigner model. In this paper all the results hold under the general assumption that  
$\mathbb{E} \xi_{ij} =0$, $\mathbb{E}\xi_{ij}^2 = 1 + O(\delta_{ij})$ and  for
all integers $p$ we have $\mathbb{E}\vert \xi_{ij}\vert^p$ finite.\footnote{The notation $O(\delta_{ij})$ means that the second moment of off-diagonal elements is $1$ but the variance of diagonal elements can be different. For example
$\xi_{ij} \sim\mathcal{N}(0, 1)$, $i<j$, and $\xi_{ii}\sim\mathcal{N}(0,2)$, corresponds to Wigner's Gaussian Orthogonal Ensemble.}

We consider the cost function ($\Vert \cdot \Vert_F$ the Frobenius norm)
\begin{equation}
\mathcal H(\theta) = \frac{1}{2 n^2}\norm{Y - \theta\theta^T}_F^2 -\frac{1}{2n^2} \norm{Y - \theta^*\theta^{*T}}_F^2 
\end{equation}
(normalized so that, $\mathcal H(\theta^*) = 0$, and at the same time, the limit $n\to +\infty$ is well defined) and want to characterize the time evolution of the estimator for $\theta^*$ provided by gradient descent dynamics on the sphere.
In gradient descent,
an initial (deterministic) vector $\theta_0 \in \mathbb{S}^{n-1}(\sqrt n)$ is updated through the autonomous ordinary differential equation
\begin{equation}\label{eq:mainode}
    \frac{d\theta_t}{dt} = - \eta \bigl( \nabla_\theta \mathcal H(\theta_t) - \frac{\theta_t}{n} \langle \theta_t, \nabla_\theta \mathcal H(\theta_t) \rangle\bigr)
\end{equation} 
where $\eta \in \mathbb R^*_+$ is a learning rate. The second term on the right hand side enforces the constraint  
$\theta_t \in \mathbb{S}^{n-1}(\sqrt n)$ at all times (see Appendix \ref{app:constraint}). 
The main quantities of interest to be computed are the time evolutions of the cost $\mathcal{H}(\theta_t)$ and overlap $q(t)=n^{-1}\langle\theta^*, \theta_{t}\rangle$ in the high-dimensional limit $n\to +\infty$. We note that the overlap is equivalent to the mean-square-error
$n^{-1}\Vert \theta_t -\theta^*\Vert_2^2 = 2\bigl(1 - \frac{\langle\theta^*, \theta_{t}\rangle}{n}\bigr)$.

{\bf Contribution:} We compute the full time evolution of the cost and overlap in the scaling limit $\lim_{n\to +\infty}\mathcal{H}(\theta_{t=\tau n/\eta})$ and 
$\lim_{n\to +\infty} \frac{\langle\theta^*, \theta_{t=\tau n/\eta}\rangle}{n}$ for all $\tau>0$. Explicit formulas are expressed solely in terms of a modified Bessel function of first order in theorems \ref{th:risktrack} and \ref{thmrisktracktrue} (section \ref{main-results}). 
The formulas allow to explore the asymptotic behavior as $\tau\to +\infty$, as well as transient behavior by computing one and two dimensional integrals numerically (section \ref{main-results}). In the long time limit we recover (analytically) as expected the phase transition at $\lambda=1$ with a limiting value of the overlap 
equal to ${\textrm sign}(\langle \theta^*, \theta_0\rangle)\sqrt{1-1/ \lambda} \mathds{1}(\lambda >1)$. 
This is the well known BBP-like phase transition found in the spectral method \cite{Pch2004TheLE, FP2006, baik2005}. The transient behavior also exhibits interesting features. For example, depending on the magnitude of the initial overlap $n^{-1}\langle \theta^*, \theta_0\rangle$ and $\lambda > 1$ for intermediate times we find that the overlap may display a maximum and then decrease to its limiting value. Such results may therefore give guidelines for applying early stopping during gradient descent to get a better estimate of the signal.
On the technical side the analysis is based on a set of integro-differential equations (derived in section \ref{integro-differential-equations}) satisfied by matrix elements of the resolvent of the noise matrix 
$\langle\theta^*, (\frac{1}{\sqrt n}\xi - z)^{-1}\theta_t\rangle$ and $\langle\theta_t, (\frac{1}{\sqrt n}\xi - z)^{-1}\theta_t\rangle$, $z\in \mathbb{C}\setminus\mathbb{R}$. These quantities concentrate with respect to the probability law of the noise matrix as $n\to +\infty$ (for deterministic $\theta^*$ and $\theta_0$). The main steps to prove concentration are explained in section \ref{concentration}. They combine concentration properties of the matrix elements of the resolvents with an adaptation of Gronwall type arguments to the integro-differential equations. Concentration of matrix elements of resolvents of random matrices amount to study the spectrum on a {\it local} scales. Such results are only a decade old in random matrix theory and go under the name of {\it local} semi-circle laws \cite{Erdoes_2008, alex2014, benaychgeorges:hal-01258444}. They have found many applications and here we provide one more.
In section \ref{analysis-gradient-descent} we present an exact analysis of the integro-differential equations and deduce the formulas for the time evolution of the overlap and cost. 

%\paragraph{Contribution:} {\color{red} We propose to re-visit two common probabilistic models by using a gradient-descent based approach. Namely, a symmetric matrix factorization problem \cite{Perry_2018} - commonly known as a spiked Wigner Model, and a 2-layer neural network to approximate a linear function - also known as a random-feature model in the literature \cite{mei2019generalization}.

{\bf Related Work:} Starting with the early work of \cite{BurerMonteiroFirst, BurerMonteiroSecond} the efficiency of gradient descent techniques has been uncovered in recent years for a host of low-rank matrix recovery modern problems, e.g., in PCA, low-rank matrix factorization, matrix completion, phase retrieval, phase synchronization, \cite{pmlr-v70-ge17a, 10.5555/3157382.3157531, 10.5555/3305381.3305509, 10.5555/3045118.3045366, pmlr-v54-park17a, ling2019landscape, pmlr-v49-bandeira16}. We also refer to \cite{Chi_2019} for a general review and references. 
Underpinning the efficiency of gradient descent in such non-convex problems, is a high-level result \cite{pmlr-v49-lee16}, stating that when the landscape satisfies a {\it strict saddle property} (i.e., critical points  are strict saddles or minima) gradient descent with sufficiently small discrete step size and random initialization will converge almost surely to a minimum \cite{pmlr-v49-lee16}. 
The spiked Wigner models falls in this category at least for $n$ finite: critical points of the cost function on the sphere $\mathcal{S}^{n-1}(\sqrt n)$ are the eigenvectors of $Y$ and it is easy to show that almost surely (with respect to the noise matrix $\xi$) the largest eigenvector is a minimum while all the other ones are strict saddles. Therefore gradient descent will converge for small enough step size to the largest eigenvector and the spectral properties of $Y$ imply that for $\lambda > 1$ with high probability this largest eigenvector has an overlap with $\theta^*$ close to $\pm \sqrt{1 -1/\lambda}$ (these known facts are briefly reviewed in Appendix 
\ref{app:landscape}).

While these approaches are able to provide guarantees and convergence rates of gradient descent and  variants thereof, they do not provide the full time-evolution and do not say much about intermediate or transient times. This is what we achieve in this paper for the admittedly simple Wigner spiked models. We believe that the techniques used here can be extended to other problems of interest in regression and learning. Recently, pure gradient descent was studied for the much harder optimization of the cost of a mixed matrix-tensor inference problem \cite{mannelli2019afraid, mannelli2019passed}  (see also \cite{Sarao_Mannelli_2020}  for Langevin dynamics) and it was shown how the structure of saddles and minima determines the phase transition thresholds. This work is based on a set of very sophisticated integro-differential CSHCK equations \cite{Crisanti1993TheSI, PhysRevLett.71.173} with a long history in the framework of Langevin dynamics on spin-glass landscapes in statistical physics. While these derivation of the CSHCK equations for the inference problem are non-rigorous and their solution entirely numerical, they contain in principle the whole time evolution of the system (in the context of spin-glasses the formalism has been made rigorous \cite{BADAGA}). Our integro-differential equations and methods are entirely different (and involve different objects) even when specializing to the matrix case, but nevertheless it might be possible to retrieve our final analytical solution by adapting the  CSHCK equations to the matrix case as done in \cite{Cugliandolo1995} for the spherical spin-glass.

{\bf Organization of the paper:} The main theorems and illustrations of analytical formulas for the whole time-evolution of the overlap and cost are presented in section \ref{main-results}. The heart of the method presented here is contained in sections \ref{integro-differential-equations} (derivation of integro-differential equations), \ref{concentration} (local semi-circle laws and concentration of solutions), \ref{analysis-gradient-descent} (analytical solution of integro-differential equations). Appendices contain proofs, of intermediate results and technical material.

In the rest of the paper it is understood that the noise matrix $\xi$ satisfies: (i) $\mathbb{E}\xi_{ij} =0$,
(ii) $\mathbb{E}\xi_{ij}^2 = 1+ O(\delta_{ij})$, (iii) $\mathbb{E}\vert \xi_{ij}\vert^p$ finite for all $p\in \mathbb{N}$.
We use the notations $H = n^{-1/2}\xi$,  $\mathbb{P}$ for its probability law, and $X_n \overset{\mathbb{P}}{ \underset{n \to\infty}{\longrightarrow}} X$ for convergence in probability, i.e., $\lim_{n\to +\infty}\mathbb{P}(\vert X_n - X\vert >\epsilon) = 0$ for any $\epsilon >0$.

\section{Analytical solutions and illustrations}\label{main-results}

We solve gradient descent dynamics \eqref{eq:mainode} in the scaling limit $t = \tau n / \eta$, with fixed $\tau > 0$ and $n\to +\infty$.\footnote{Equivalently this corresponds to solve \eqref{eq:mainode} for a learning rate $\eta=n$, or if we would work with discrete time steps, these would be of order $1/n$. This is the order of magnitude  time steps in numerical experiments in paragraph \ref{subsubsec:experiments}.} 
The main quantities that we determine in the scaling limit are the overlap $q(\tau) = \frac{1}{n} \langle\theta^*, \theta_{n\tau/\eta}\rangle$ and the cost $\mathcal{H}(\theta_{n\tau/\eta})$. We remark that the overlap is directly linked to the mean-square error $n^{-1}\Vert \theta^* - \theta_{n\tau/\eta}\Vert^2 = 2(1 - q(\tau))$.

The initial condition $\theta_0$ is fixed such that $q(0) =\alpha$ where $\alpha\in [-1, 1]$ is independent of $n$. It will become clear that: (i) If $\theta_t$ is a solution with initial condition $q(0) = \alpha$ then $ - \theta_t$ is a solution with
$q(0) = -\alpha$; (ii) For $\alpha=0$ the solution remains trivial $q(\tau)= 0$. Therefore the reader can keep in mind that $\alpha > 0$ (all the analysis is valid for any $\alpha$ though). 

\subsection{Main results}

The solution of the gradient descent dynamics can be entirely expressed thanks to a scaled moment generating function of Wigner's semi-circle law $\mu_{\rm sc}(s) = \frac{1}{2\pi}\sqrt{4 -s^2}\, \chi_{[-2, 2]}(s)$,
\begin{align}\label{eq:MGF}
M_\lambda(\tau) = \int_{-\infty}^{\infty} ds \, \mu_{\rm sc}(s) e^{s \frac{\tau}{\sqrt\lambda}}
\end{align}
Setting $s= 2\cos\theta$ we have $M_\lambda(\tau) = 2\int_0^\pi \frac{d\theta}{\pi} (\sin\theta)^2 e^{\frac{2\tau}{\sqrt\lambda} \cos\theta}$. Integration by parts then shows that $M_\lambda(\tau) = \frac{\sqrt\lambda}{\tau} I_1(\frac{2\tau}{\sqrt{\lambda}})$ where 
$I_1(x) = \int_0^\pi \frac{d\theta}{\pi} (\cos\theta) e^{x\cos\theta}$ is a modified Bessel function of the first kind. 

\begin{theorem}[Time evolution of the overlap]\label{th:risktrack}
Let $\theta_0\in \mathbb{S}^{n-1}(\sqrt  n)$ an initial condition such that $q(0) = \frac{1}{n} \langle\theta^*, \theta_{0}\rangle = \alpha$ for a fixed $\alpha \in [-1, +1]$. The overlap converges in probability to a deterministic limit:
\begin{align}
q(\tau) \overset{\mathbb{P}}{ \underset{n \to\infty}{\longrightarrow}} \bar q(\tau) = \frac{\hat q(\tau)}{\sqrt{\hat p(\tau)}}
\end{align}
where 
\begin{align}
\hat q(\tau) = 
\alpha e^{(1+\frac{1}{\lambda}) \tau} 
\bigl[ 1 - \frac{1}{\lambda} \int_0^\tau ds\, e^{-(1+\frac{1}{\lambda}) s} M_\lambda(s) \bigr] 
\label{eq:hatq0}
\end{align}
and 
\begin{align}
\hat p(\tau) = M_\lambda(2\tau) + 2 \alpha \int_0^\tau ds\, \hat q(s) M_\lambda(2\tau-s)
        + \int_0^\tau\int_0^\tau dudv\, \hat q(u) \hat q(v) M_\lambda(2\tau-u-v).
\label{eq:hatp0}
\end{align}
\end{theorem}

\begin{theorem}[Time evolution of the cost]\label{thmrisktracktrue}
Under the same conditions as in theorem \ref{th:risktrack} the cost converges to a deterministic limit
$\mathcal{H}(\theta_{\tau n/\eta}) \overset{\mathbb{P}}{ \underset{n \to\infty}{\longrightarrow}}  1 - \frac{1}{2} \frac{d}{d\tau}\big\{\ln \hat p(\tau)\big\}$.
\end{theorem}

Using asymptotic properties of the Bessel function and the Laplace method it is possible to calculate the 
asymptotics of the integrals in \eqref{eq:hatq0} and \eqref{eq:hatp0} for $\tau\to +\infty$. We find for the overlap
$\lim_{\tau \to \infty} \bar q(\tau) = {\rm sign}(\alpha) \sqrt{1 - \lambda^{-1}} \mathds{1}(\lambda \geq 1)$. The overlap displays the well known phase transition at $\lambda=1$ also predicted by the spectral method. The asymptotic values can also be derived independently from theorem \ref{th:risktrack} by directly looking at the stationary equation $\nabla_\theta \mathcal H(\theta_\infty) - \frac{\theta_\infty}{n} \langle \theta_\infty, \nabla_\theta \mathcal H(\theta_\infty) = 0$. This is discussed in Appendix \ref{app:limiting-behavior-stoc} for completeness.
It is also possible to go one step further in the asymptotics to argue that at the transition $\lambda =1$ the power law behavior holds $q(t) \sim (\frac{2}{\pi\tau})^{1/4}$ (see Appendix \ref{app:asymptotic}). 

%\begin{corollary}[Asymptotic values of overlap and cost]
%Under the same conditions as in theorem \ref{th:risktrack} we find for the overlap
%$\lim_{\tau \to \infty} \bar q_0(\tau) = {\rm sign}(\alpha) \sqrt{1 - \lambda^{-1}} \mathds{1}(\lambda \geq 1)
%$
%and for the cost

Besides the transition at $\lambda = 1$, for finite $\lambda$, a detailed analysis of the equations of theorem \ref{th:risktrack} which are described also in Appendix \ref{app:asymptotic} allows to derive the first order asymptotic behavior of $\bar q$ for large $\tau$. These tedious calculations are carried out analytically in detail and checked numerically. 
Specifically, in the regime $1 < \lambda < +\infty$ we find
\begin{equation}
    \bar q(\tau) - \text{sign}(\alpha) \sqrt{1-\frac{1}{\lambda}}
    \sim\frac{  \text{sign}(\alpha)  
        }{2 \sqrt{\pi} \lambda^\frac14 \sqrt{1-\frac{1}{\lambda}}  \left(1 - \frac{1}{\sqrt{\lambda}}\right)^2}  
        \tau^{-\frac32}  e^{ -(1-\frac{1}{\sqrt \lambda})^2 \tau } 
    \label{asymp_pos_initial}
\end{equation}
As for $0 < \lambda < 1$, we retrieve a power law behavior:
\begin{equation}
    %\label{asym_neg}
    \bar q(\tau)
    \sim  
    \frac{ \alpha \left( \frac{2}{\pi} \right)^\frac14  }{
        \lambda^\frac58 
        \left(1 - \frac{1}{\sqrt{\lambda}}\right)^2
        \sqrt{
        1 - \alpha^2 + \frac{\alpha^2}{\lambda (\frac{1}{\sqrt \lambda} - 1)^2 }
    }} \tau^{- \frac34}
\end{equation}

The noise-less regime $\lambda = +\infty$ is an elementary case for which the overlap can be obtained very simply. Taking the inner product of \eqref{eq:mainode} with $\theta^*$ we find the differential equation (for $t=\tau n/\eta$)
$\frac{dq(\tau)}{d\tau} = q(\tau) - q(\tau)^3$, $q(0) = \alpha$,
which has the solution
$q(\tau) = \alpha (\alpha^2 + (1-\alpha^2) e^{-2\tau})^{-1/2}$.
As we will see, in the noisy case there is no closed first order ODE for $q(\tau)$ and we must solve integro-differential equations for suitable generating functions (or an an infinite hierarchy of coupled differential equations for generalized overlaps).
As a sanity check, we can verify that theorem \ref{th:risktrack} leads to the same expression when $\lambda\to +\infty$. Explicitly, we find
$\lim_{\lambda \to +\infty} \hat q(\tau) = \alpha e^\tau$ and 
$\lim_{\lambda \to +\infty} \hat p(\tau) = 1 - \alpha^2 + \alpha^2 e^{2\tau}$ which implies the noiseless expression for 
the overlap.

\subsection{Discussion and numerical experiments}\label{subsec:numerics}

Theorems \ref{th:risktrack} and \ref{thmrisktracktrue} provide theoretical predictions for the full time evolution of the overlap and risk in the high dimensional limit $n\to +\infty$. In this section (and Appendix \ref{app:add-experiments}) we briefly illustrate and discuss this time evolution. Moreover in Appendix \ref{app:add-experiments} we also compare the theoretical predictions with simulations of discrete step size gradient descent for runs over multiple samples of $\xi$. 

\subsubsection{Time evolution of the overlap}

%Using a standard desktop setting with Scipy library in Python, 
Figure \ref{fig:theoretical_curves} shows the theoretical overlap at all times $\tau \in \mathbb R^+$ for  two initial conditions $\alpha = 0.1$ and $\alpha= 0.5$ and any signal-to-noise ratio $\lambda$. Let us say a few words about the {\it transient behaviors} that are observed.
On the one hand, the closer $\alpha$ gets to $0$, the longer it takes for the gradient descent to "kick-in": the overlap stays longer close to $0$ before reaching its asymptotic behavior.
An additional example for $\alpha = 0.01$ illustrates this fact in Appendix \ref{app:add-experiments}. On the other hand, we clearly see that when the initial overlap $\alpha$ is not too close to $0$, the time evolution is {\it not} monotone even for $\lambda > 1$, and a specific bump is reached at early times where the overlap reaches a maximum before dropping down to its limit. In fact this is clearly suggested by \eqref{asymp_pos_initial} for $\alpha <1$. This can be seen in particular in the case $\alpha = 0.5$ in figure \ref{fig:theoretical_curves} (b). This suggests that in practice, in such situations, it may be worth using early-stopping techniques to optimize the estimation of the signal.

%
% \begin{figure}
%     \centering
%     \subfigure[$\alpha = 0.1$]{\includegraphics[width=7cm]{./images/theoretical/experiment_2.png} }
%     \qquad
%     \subfigure[$\alpha = 0.5$]{\includegraphics[width=7cm]{./images/theoretical/experiment_3.png} }
%     \caption{Overlap as a function of time according to theorem \ref{th:risktrack} for two initial conditions and different 
%     signal to noise ratios. Thick dotted line corresponds to $\lambda =1$ and tends to zero slowly as $(2/\pi\tau)^{1/4}$. For $\lambda <1$ the curves tend to zero and for $\lambda >1$ they tend to $\sqrt{1 - 1/\lambda}$. }
%     \label{fig:theoretical_curves}
% \end{figure}
%
%It is worth mentioning that $\alpha$ is expected to be small at the beginning when $\theta_0$ and $\theta^*$ are chosen randomly. 

In the case $\lambda = 1$ one can show that $\hat q(\tau) = \alpha \left( I_0(2\tau) + I_1(2\tau) \right)$ (with modified Bessel functions of the first kind) and it is numerically much easier to evaluate the asymptotic behavior of $q(\tau)$. The calculation yields
$q(\tau) \sim \left(\frac{2}{\pi \tau}\right)^{\frac14}$ (see Appendix \ref{app:asymptotic}).
Furthermore plotting a family of curves with $\lambda = 1$ and $\alpha \in (0,1)$ in figure \ref{fig:special_experiment}, it appears that this asymptote also seems to act as an {\it upper-bound}.
% \begin{figure}
%     \centering
%     \includegraphics[width=7cm]{./images/theoretical/special_experiment.png} 
%     \caption{Overlap comparison over time with different for $\lambda=1$ with a range of values for $\alpha$}
%     \label{fig:special_experiment}
% \end{figure}

\subsubsection{Time evolution of the cost}
We also have predictions for the evolution of cost at any time for any values of $(\alpha, \lambda)$. This is illustated in figure \ref{fig:special_experiment2}. As seen in the analysis of section \ref{cost-analysis}, Equ. \eqref{eq:Rqp} the cost has two additive contributions basically interpreted as $q(\tau)^2$ and $p_1(\tau) = n^{-1}\langle \theta_\tau, H \theta_\tau\rangle$. The second contribution equals $n^{-1}{\rm Tr} H \theta_\tau\theta_\tau^T$ can be interpreted as a similarity measure of the reconstructed matrix 
$\theta_\tau\theta\tau^T$ and the noise matrix $H$, and is thus a "proxy" for assessing over-fitting in this particular setting.
Interestingly, in the depicted example where $\lambda=2, \alpha=0.1$, $p_1(\tau)$ is shown to decrease the risk at early stages at a fast rate, until it slightly "heals" for $\tau \geq 3$. Conversely, when $\alpha = 0.5$, we see $p_1(\tau)$ does not decrease as much in early stages, and the healing phenomenon does not occur. At the same time, as observed on \ref{fig:theoretical_curves} (b)
$q(\tau)$ is not monotonous: it increases at early stages and decreases down to its limiting value later.
% \begin{figure}
%     \centering
%     \subfigure[$\alpha = 0.1$]{\includegraphics[width=7cm]{./images/cost_analysis/analysis_2.png} }
%     \qquad
%     \subfigure[$\alpha = 0.5$]{\includegraphics[width=7cm]{./images/cost_analysis/analysis_5.png} }
%     \caption{Cost evolution for $\lambda=5$}
%     \label{fig:special_experiment2}
% \end{figure}

\begin{figure}
    \centering
    \subfigure[$\alpha = 0.1$]{\includegraphics[width=7cm]{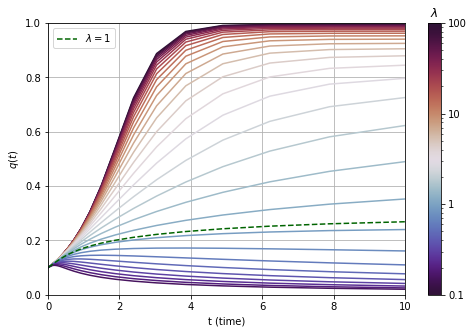} }
    \qquad
    \subfigure[$\alpha = 0.5$]{\includegraphics[width=7cm]{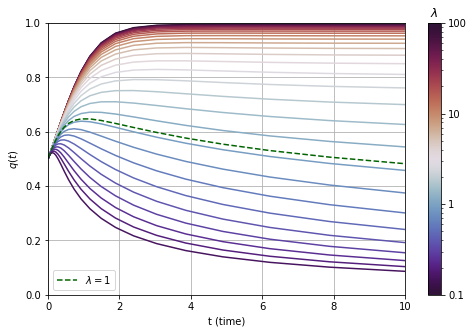} }
    \caption{Overlap as a function of time according to theorem \ref{th:risktrack} for two initial conditions and different 
    signal-to-noise ratios. Thick dotted line corresponds to $\lambda =1$ and tends to zero slowly as $(2/\pi\tau)^{1/4}$. For $\lambda <1$ the curves tend to zero and for $\lambda >1$ they tend to $\sqrt{1 - 1/\lambda}$. }
    \label{fig:theoretical_curves}
\end{figure}
\begin{figure}
    \centering
    \includegraphics[width=7cm]{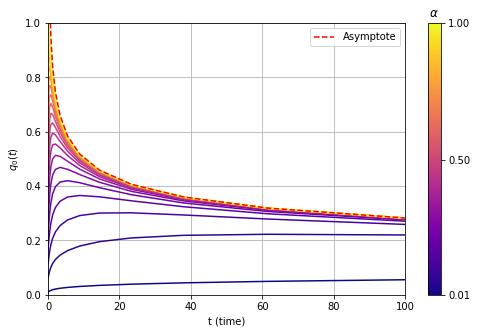} 
    \caption{Overlap comparison for $\lambda=1$ with a range of values for $\alpha$}
    \label{fig:special_experiment}
\end{figure}
\begin{figure}
    \centering
    \subfigure[$\alpha = 0.1$]{\includegraphics[width=7cm]{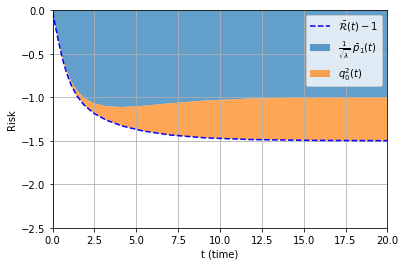} }
    \qquad
    \subfigure[$\alpha = 0.5$]{\includegraphics[width=7cm]{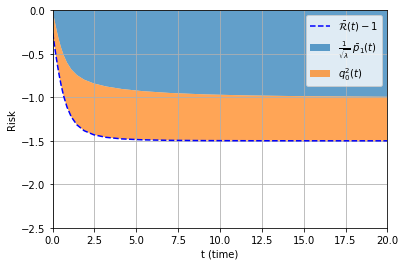} }
    \caption{Cost evolution for $\lambda=5$}
    \label{fig:special_experiment2}
\end{figure}

% question : est-ce que "commencer à alpha=0.5" c'est la même chose que "commencer à alpha=0.1, attendre que q0=0.5, puis comparer les courbes à partir de là"?

%\paragraph{Dynamic learning-rate example - RMSProp:} Root mean square propagation is arguably on of the most popular alternative to a simple gradient descent algorithm. At a continuous level, we consider the following ODE:
%\begin{equation}
%    \left\{
%    \begin{matrix}
%        \partial_t v_t(\theta) = \gamma \left(\norm{ \nabla_\theta H(\theta)}^2 - v_t(\theta) \right)\\
%        \partial_t \theta_t = 
%        - \frac{\eta}{\sqrt{v_t (\theta_t)}} 
%        \left( 
%            \nabla_\theta H(\theta_t)
%            - \lambda(\theta_t) \theta_t
%        \right)
%    \end{matrix}\right.
%\end{equation}
%A discretized version translates directly into:
%\begin{equation}
%    \left\{
%    \begin{matrix}
%        v_{t+\delta t}(\theta) = (1 - \gamma) \delta t v_t(\theta) + \gamma \delta t \norm{ \nabla_\theta H(\theta)}^2\\
%        \theta_{t + \frac{\delta t}{2}} = \theta_t
%        - \frac{\eta \delta t}{\sqrt{v_t (\theta_t)}} \nabla_\theta H(\theta_t)
%    \end{matrix}\right.
%\end{equation}

%\input{integro-differential-equations-ancien}
\section{Integro-differential equations}\label{integro-differential-equations}

We study gradient descent in a regime where $t = \tau n/\eta$, $n\to +\infty$, with $\tau$ fixed. Abusing slightly notation we set $\theta_{\tau n/\eta} \to \theta_{\tau}$ so that
equation \eqref{eq:mainode} reads
\begin{align}\label{eq:mainode2}
\frac{d\theta_\tau}{d\tau}  
= -  n \nabla_\theta \mathcal H(\theta_\tau) + \theta_\tau \langle \theta_\tau, \nabla_\theta \mathcal H(\theta_\tau) \rangle
%\nonumber \\ &
=
\frac{1}{n^2} Y \theta_\tau - \frac{1}{n^2} \langle \theta_\tau, Y \theta_\tau \rangle \theta_\tau
\end{align}
We define $H = n^{-1/2} \xi$ the suitably normalized noise matrix. Besides the basic overlap $q(\tau) = \frac{1}{n}\langle \theta^*, \theta_\tau\rangle$, another one also plays an important role, namely  $p_1(\tau) = \frac{1}{n} \langle\theta_\tau, H \theta_\tau\rangle$.\\
Using $Y = \theta^* \theta^{*T} + \frac{n}{\sqrt{\lambda}} H$ we find
\begin{align}
\frac{d \theta_\tau}{d\tau}  
%    \frac{\langle \theta_\tau, \theta^*\rangle}{n}  \theta^*
%    + \frac{1}{\sqrt \lambda} H \theta_\tau
%    - \left( \frac{\langle \theta_\tau, \theta^* \rangle^2}{n^2} + 
%      \frac{1}{\sqrt \lambda} \frac{\langle \theta_\tau, H \theta_\tau\rangle}{n}
%    \right) \theta_\tau
%    \nonumber \\ &
    =
    q(\tau)\theta^* + \frac{1}{\sqrt \lambda} H \theta_\tau  - \left( q(\tau)^2 + 
      \frac{p_1(\tau)}{\sqrt \lambda}
    \right) \theta_\tau
    \label{eq:continuity}
\end{align}
It is not possible to write down a closed set of equations that involve only $q(\tau)$ and $p_1(\tau)$, but only for a hierarchy of such objects, or for their generating functions. We now introduce these generating functions and then give the closed set of equations which they satisfy.

The $n\times n$ matrix $H = n^{-1/2} \xi$ is drawn with the probability law $\mathbb{P}$. 
Fix any small $\delta >0$ and let 
$\mathcal{S}_\delta^{n}$ the set of realizations of $H$ such that all eigenvalues fall in an interval $I_\delta =[-2 - \delta, 2 + \delta]$. Then $\mathbb{P}(\mathcal{S}_\delta^{n}) \to 1$ as $n\to +\infty$ (see for example \cite{Erdoes2011Survey}). In the rest of this section it is understood that 
$H\in \mathcal{S}_\delta^{n}$. 
In particular the resolvent matrix\footnote{Here $I$ is the identity $n\times n$ matrix and we will slightly abuse notation by omitting it and simply write $(H -z)^{-1}$.}
$\mathcal{R}(z) = (H - z I)^{-1}$ is well defined for 
$z\in\mathbb{C}\setminus I_\delta$ if $H\in \mathcal{S}_\delta^{n}$. 

For 
any contour $\mathcal{C} = \{z\in \mathbb{C} \mid z= \rho e^{i\theta}, \theta\in [0, 2\pi]\}$ with $\rho >2 + \delta$ we can define three generating functions
\begin{align}\label{eq:genfunctions}
Q_\tau(z) = \frac{1}{n}\langle \theta_\tau, \mathcal{R}(z) \theta^* \rangle, 
\quad
P_\tau(z) = \frac{1}{n}\langle \theta_\tau, \mathcal{R}(z) \theta_\tau \rangle,
\quad
R(z) = \frac{1}{n}\langle \theta^*, \mathcal{R}(z) \theta^* \rangle .
\end{align}
From standard holomorphic functional calculus for matrices (see for example \cite{schwartz1958linear}) we have 
\begin{align}
q(\tau)  = - \oint_{\mathcal{C}} \frac{dz}{2\pi i} Q_\tau(z), \quad 
p_1(\tau)  = - \oint_{\mathcal{C}} \frac{dz}{2\pi i} z P_\tau(z) .
\end{align}
Note that these two overlaps are part of a hierarchy of overlaps $q_k(\tau) \equiv \frac{\langle \theta^*, H^k\theta_\tau \rangle}{n} = - \oint_{\mathcal{C}} \frac{dz}{2\pi i}  z^k Q_\tau(z)$ and $p_k(\tau) \equiv \frac{\langle \theta_\tau, H^k \theta_\tau \rangle}{n} = - \oint_{\mathcal{C}} \frac{dz}{2\pi i} z^k P_\tau(z)$, $k\geq 1$, which can all be calculated by the methods of this paper (note $q(\tau)$ corresponds to $k=0$). 

\begin{proposition}\label{diff-equ-gen}
For any realization $H \in \mathcal{S}_\delta$ and any $z\in \mathbb{C}\setminus I_\delta$ the generating functions \eqref{eq:genfunctions} satisfy the integro-differential equation
\begin{align}\label{eq:integrodiff}
\begin{cases}
    \frac{d}{d\tau}Q_\tau(z) = q(\tau) R(z) 
                               + \frac{1}{\sqrt \lambda} (z Q_\tau(z) + q(\tau))
                                - \left(q^2(\tau) + \frac{1}{\sqrt \lambda} p_1(\tau) \right) Q_\tau(z) 
    \\
    \frac12 \frac{d}{d\tau} P_\tau(z) = q(\tau) Q_\tau(z) 
                                        + \frac{1}{\sqrt \lambda} (z P_\tau(z) + 1)
                                         - \left(q^2(\tau) + \frac{1}{\sqrt \lambda} p_1(\tau) \right) P_\tau(z)
\end{cases}
\end{align}
where $q(\tau) = -\oint_{\mathcal{C}} \frac{dz}{2\pi i} Q_\tau(z)$ and $p_1(\tau) = -\oint_{\mathcal{C}} \frac{dz}{2\pi i} z P_\tau(z)$. 
%Using \eqref{eq:recovermoments} we immedialtely get back \eqref{eq:odes} for $q_k(\tau), p_k(\tau), r_k(\tau)$. In particular we obtain 
%the overlap $q_0(\tau)$.
\end{proposition}
%\begin{proof}
%The derivation of \eqref{eq:integrodiff} is a direct application of \eqref{eq:continuity}.
%Details are found in appendix \ref{app:diff-equ-genfunction}.
%\end{proof}

\begin{proof}
Let us derive the first equation. Using \eqref{eq:continuity}
\begin{align}
\frac{d}{d\tau} Q_\tau(z) 
 = 
\frac{1}{n}\langle\theta^*, (H -z)^{-1} \frac{d\theta_\tau}{d\tau}\rangle
%\nonumber \\ &
&
=
\frac{q(\tau)}{n} \langle \theta^*, (H -z)^{-1} \theta^*\rangle 
+ 
\frac{1}{n \sqrt \lambda} \langle \theta^*, (H -z)^{-1}H \theta_\tau\rangle 
\nonumber \\ & 
-
\left( q(\tau)^2 + 
      \frac{p_1(\tau)}{\sqrt \lambda}
    \right) \frac{1}{n}\langle \theta^*, (H -z)^{-1}\theta_\tau\rangle
\end{align}
Using $(H -z)^{-1}H  = I + z (H-z)^{-1}$ in the second term in the right hand side, we immediately get the first equation of \eqref{eq:integrodiff}.
Let us now derive the second equation. Again using \eqref{eq:continuity}
\begin{align}
\frac{d}{d\tau} P_\tau(z) 
& = 
\frac{1}{n}\langle \frac{d\theta_\tau}{d\tau}, (H -z)^{-1} \theta_\tau\rangle + \frac{1}{n}\langle \theta_\tau, (H -z)^{-1} \frac{d\theta_\tau}{d\tau}\rangle
\nonumber \\ &
= \frac{q(\tau)}{n} \langle \theta^*, (H - z)^{-1} \theta_\tau\rangle + \frac{q(\tau)}{n} \langle \theta_\tau, (H - z)^{-1} \theta^*\rangle
%\nonumber \\ &
%\,\,\,\,\,\,  
+
\frac{1}{n\sqrt \lambda} \langle H \theta_\tau, (H-z)^{-1} \theta_\tau\rangle  
\nonumber \\ &
+ \frac{1}{n\sqrt \lambda} \langle \theta_\tau, (H-z)^{-1}H \theta_\tau\rangle 
%\nonumber \\ &
%\,\,\,\,\,\,  - 
-
2 \left( q(\tau)^2 + \frac{p_1(\tau)}{\sqrt \lambda}\right) \frac{1}{n}\langle\theta_\tau, (H-z)^{-1} \theta_\tau\rangle
\end{align}
Since $(H - zI)$ is a symmetric matrix $\langle \theta^*, (H - z)^{-1} \theta_\tau\rangle = \langle \theta_\tau, (H - z)^{-1} \theta^*\rangle$ and $\langle H \theta_\tau, (H-z)^{-1} \theta_\tau\rangle = \langle \theta_\tau, (H-z)^{-1}H \theta_\tau\rangle$. Thus using again 
$(H -z)^{-1}H  = I + z (H-z)^{-1}$ and $\langle\theta_\tau, \theta_\tau\rangle = 1$ we get the second equation of \eqref{eq:integrodiff}.
\end{proof}

\section{Concentration results}\label{concentration}

We introduce the Stieltjes transform of the semi-circle law $\mu_{\text{sc}}(s) = \frac{1}{2\pi}\sqrt{4- s^2} \chi_{[-2, 2]}(s)$,
\begin{align}\label{eq:Stieltjes}
G_{\text{sc}}(z) = \int_{\mathbb{R}} ds \frac{\mu_{\text{sc}}(s) }{s -z} = \frac{1}{2}( -z + \sqrt{z^2 -4}), \quad z\in \mathbb{C}\setminus [-2, 2].
\end{align}
It is a classical result of random matrix theory \cite{Erdoes2011Survey} that, for any $z\in \mathbb{C}\setminus [-2, 2]$,
$$\frac{1}{n} {\rm Tr} \mathcal{R}(z) \overset{\mathbb{P}}{ \underset{n \to\infty}{\longrightarrow}} G_{\text{sc}}(z)$$. However here we will need convergence in probability of {\it matrix elements} of the resolvent (for given $z$ and also uniformly in $z$). This tool is provided by recent results in random matrix theory that go under the name of {\it local semi-circle laws} \cite{alex2014}. 

Recall that $\mathcal{S}_\delta^n$ is the set of realizations of $H =\frac{1}{\sqrt n} \xi$ with eigenvalues in $I_\delta = [-2 -\delta, 2+\delta]$, $\delta >0$, and that $\lim_{n\to +\infty}\mathbb{P}(\mathcal{S}_\delta^n) = 1$. It will be convenient to use the notation 
$\mathbb{P}_\delta$ for the conditional probability law of $H$ {\it conditioned} on the event  $H\in \mathcal{S}_\delta^n$.

\subsection{Initial condition analysis}\label{subsec:initial}

%We consider also the subset of events $\Omega^\alpha = \left\{\omega \in \Omega | \frac{1}{n} \langle \theta_0, \theta^* \rangle = \alpha \right\}$ with the probability measure $\mathcal P^\alpha = \mathcal P(\cdot | \Omega^\alpha)$. In other words, we fix the initial overlap to a predefined constant $\alpha$ which can be chosen such that $0<|\alpha|<1$. The effect of this condition is clear, as $\lim_{n} \frac{1}{n} \langle \theta_0, \theta^* \rangle = 0 ~\mathcal P\text{ - a.s}$ whereas
%$\lim_{n} \frac{1}{n} \langle \theta_0, \theta^* \rangle = \alpha ~\mathcal P^\alpha\text{ - a.s}$. The first condition is not of interest as this would lead to a limiting $\bar Q_0(0) = 0 ~\mathcal P\text{ - a.s}$ which wouldn't result in any evolution.

We first derive natural initial conditions for the integro-differential equations \eqref{eq:integrodiff} when $\frac{1}{n}\langle\theta_0, \theta^*\rangle = q(0) = \alpha \in [-1, +1]$.
%{\color{blue} It is possible to check that an initial condition $\theta_0$ orthogonal to $\theta^*$, i.e., $q_0(0) = 0$ leads to $q_0(\tau) = 0$ for all $\tau \geq 0$.} To obtain a non-trivial evolution we must take a finite initial overlap. For $\alpha >0$ a fixed constant we set a (deterministic) initial condition such that $\frac{1}{n}\langle\theta_0, \theta^*\rangle = q_0(0) = \alpha$. 
We claim (corollary \ref{le_mse} below) that the initial conditions $Q_0(z), P_0(z)$ as well as $R(z)$ concentrate on explicit functions $\bar Q_0(z), \bar P_0(z), \bar R(z)$. The main tool is the following proposition which we prove in 
section \ref{app:initial-conditions} (based on a theorem in \cite{alex2014}):
\begin{proposition}\label{local-semi-circle-version}
Fix $\delta >0$, $\epsilon>0$.
For any fixed $z \in \mathbb C \setminus I_\delta$ and any deterministic sequence of unit vectors $u^{(n)},v^{(n)} \in \mathbb{S}^{n-1}(1)$ the $n$-sphere of unit radius, we have
\begin{equation}
        \lim_{n \to \infty} \mathbb{P}_\delta\bigl( 
            \vert \langle u^{(n)}, \mathcal R(z) v^{(n)} \rangle - \langle u^{(n)}, v^{(n)}\rangle G_{\text{sc}}(z) \vert
            > \epsilon
        \bigr) = 0.
    \end{equation}
\end{proposition}
%Therefore, the previous theorem states that for a predefined $z$, we have convergence in probability of the random variable $\langle u, \mathcal R(z) v\rangle$ as $n$ goes to infinity towards the constant $\langle u, v \rangle G_{\text{sc}(z)}$.
Applying this proposition to the three pairs of unit vectors $(u^{(n)}, v^{(n)}) = (\frac{\theta_0}{\sqrt n}, \frac{\theta^*}{\sqrt n})$, $(\frac{\theta_0}{\sqrt n}, \frac{\theta_0}{\sqrt n})$, and $(\frac{\theta^*}{\sqrt n}, \frac{\theta^*}{\sqrt n})$ we directly obtain

\begin{corollary}\label{le_mse}
    Fix $\alpha\in [-1, +1]$ and $\theta_0$ such that $\frac{1}{n}\langle \theta_0, \theta^*\rangle = \alpha$.
    For $z \in \mathbb C \setminus \mathbb{I}_\delta$ we have convergence in probability of $Q_0(z), P_0(z), R(z)$ to the Stieljes transform of the semi-circle law:
    \begin{align}
            Q_0(z) \overset{\mathbb{P}_\delta}{ \underset{n \to\infty}{\longrightarrow}} \bar Q_0(z) = \alpha G_{\text{sc}}(z), \quad
            P_0(z) \overset{\mathbb{P}_\delta}{ \underset{n \to\infty}{\longrightarrow}} \bar P_0(z) = G_{\text{sc}}(z), \quad
            R(z) \overset{\mathbb{P}_\delta}{ \underset{n \to\infty}{\longrightarrow}} \bar R(z) = G_{\text{sc}}(z). 
    \end{align}
\end{corollary}

\subsection{Concentration of the overlap for finite times}\label{subsec:conc}

We consider the integro-differential equations \eqref{eq:integrodiff} for the limiting initial conditions $(\bar Q_0(z), \bar P_0(z)) = (\alpha G_{\text{sc}}(z), G_{\text{sc}}(z))$ and limiting $\bar R(r) = G_{\text{sc}}(z)$. 
More  explicitly we define $\bar Q_\tau(z)$, $\bar P_\tau(z)$ as the (holomorphic over $z\in\mathbb{C}\setminus I_\delta$) solutions of
\begin{align}\label{eq:limitingdiff}
\begin{cases}
    \frac{d}{d\tau}\bar Q_\tau(z) = \bar q(\tau) \bar R(z) 
                               + \frac{1}{\sqrt \lambda} (z \bar Q_\tau(z) + \bar q(\tau))
                                - \left(\bar q^2(\tau) + \frac{1}{\sqrt \lambda} \bar p_1(\tau) \right) \bar Q_\tau(z) 
    \\
    \frac12 \frac{d}{d\tau} \bar P_\tau(z) = \bar q(\tau) \bar Q_\tau(z) 
                                        + \frac{1}{\sqrt \lambda} (z \bar P_\tau(z) + 1)
                                         - \left(\bar q^2(\tau) + \frac{1}{\sqrt \lambda} \bar p_1(\tau) \right) \bar P_\tau(z)
\end{cases}
\end{align}
where {\it by definition} $\bar q(\tau) = -\oint_{\mathcal{C}} \frac{dz}{2\pi i} \bar Q_\tau(z)$ and $\bar p_1(\tau) = -\oint_{\mathcal{C}} \frac{dz}{2\pi i} z \bar P_\tau(z)$, and the initial conditions are $\bar Q_0(z) = \alpha G_{\rm sc}(z)$, $\bar P_0(z) = G_{\rm sc}(z)$.
The explicit calculation of the solutions 
$\bar Q_\tau(z)$, $\bar P_\tau(z)$ 
in section \ref{analysis-gradient-descent} shows that they exist and they are holomorphic for 
$z\in\mathbb{C}\setminus I_\delta$. 

One can show that the concentration result of corollary \ref{le_mse}  
%drole de label ici de
extends to all finite times.
This can be done by a Gr\"onwall stability type argument. A difficulty with respect to the standard argument is that here we deal with an integro-differential equation instead of purely ordinary differential equation. For this reason we need a {\it uniform} (over $z$) concentration result which strengthens proposition \ref{local-semi-circle-version}. The following is proved in section \ref{app:initial-conditions}.

\begin{proposition}\label{uniform-local-semi-circle-version}
Fix $\delta>0$, $\epsilon>0$. 
Recall $\mathcal{C} = \{z\in \mathbb{C} \mid z = \rho e^{i\theta}, \theta\in [0, 2\pi]\}$ for $\rho \geq 2 + \delta$.  For any deterministic sequence of unit vectors $u^{(n)},v^{(n)} \in \mathbb{S}^{n-1}(1)$ the $n$-sphere of unit radius, we have
    \begin{equation}
        \lim_{n \to \infty} \mathbb{P}_\delta\bigl( \sup_{z\in \mathcal{C}}
            \vert \langle u^{(n)}, \mathcal R(z) v^{(n)} \rangle - \langle u^{(n)}, v^{(n)}\rangle G_{\text{sc}}(z) \vert
            > \epsilon
        \bigr) = 0.
    \end{equation}
\end{proposition}

Applying this proposition to appropriate pairs of unit vectors as previously we get directly:

\begin{corollary}\label{stronger}
Fix $\alpha \in [-1, +1]$ and $\theta_0$ such that $\frac{1}{n}\langle \theta_0, \theta^*\rangle = \alpha$. Let 
$\mathcal{C} = \{z\in \mathbb{C} \mid z = \rho e^{i\theta}, \theta\in [0, 2\pi]\}$ for some $\rho \geq 2 + \delta$.
Recall $\bar Q_0(z) = \alpha G_{\text{sc}}(z)$, $\bar P_0(z) = G_{\text{sc}}(z)$, $\bar R(z) = G_{\text{sc}}(z)$. Then 
$\sup_{z\in \mathcal{C}}\vert Q_0(z) - \bar Q_0(z)\vert$,  $\sup_{z\in \mathcal{C}}\vert P_0(z) - \bar P_0(z)\vert$,
$\sup_{z\in \mathcal{C}}\vert R(z) - \bar R(z)\vert$ all converge in $\mathbb{P}_\delta$-probability to zero.
\end{corollary}

In section \ref{app:proof-gronwall} this corollary is used to prove:

\begin{proposition}\label{concentration-a-tout-temps-fini}
Fix $\alpha\in [-1, +1]$ and $\theta_0$ such that $\frac{1}{n}\langle \theta_0, \theta^*\rangle = \alpha$. Fix any $T >0$.
We have convergences of the following overlaps to the deterministic limits
$q(\tau) \overset{\mathbb{P}}{ \underset{n \to\infty}{\longrightarrow}} \bar q(\tau)$, 
$p_1(\tau) \overset{\mathbb{P}}{ \underset{n \to\infty}{\longrightarrow}} \bar p_1(\tau)$ where here convergence is with respect to the probability law $\mathbb{P}$ of the GOE.
\end{proposition}

\begin{remark}
With a bit more work the proof of this corollary can be strengthened to also show that for any $z \in \mathbb{C} \setminus I_\delta$ and $\tau\in [0, T]$ we have convergence in probability of $Q_\tau(z), P_\tau(z)$ to the deterministic solutions of the integro-differential equations \eqref{eq:limitingdiff}, i.e., 
$Q_\tau(z) \overset{\mathbb{P}_\delta}{ \underset{n \to\infty}{\longrightarrow}} \bar Q_\tau(z)$, 
$P_\tau(z) \overset{\mathbb{P}_\delta}{ \underset{n \to\infty}{\longrightarrow}} \bar P_\tau(z)$,
as well as convergence of all overlaps 
$q_k(\tau), p_k(\tau) \overset{\mathbb{P}}{ \underset{n \to\infty}{\longrightarrow}} \bar q_k(\tau), \bar p_k(\tau)$ ($k\geq 1$).
Since we will not need these results we omit their proof.
\end{remark}

\section{Solution of integro-differential equations and overlap}\label{analysis-gradient-descent}

In this section we analyze \eqref{eq:limitingdiff} for $z\in\mathbb{C}\setminus I_\delta$ with the initial conditions $(\bar Q_0(z), \bar P_0(z), \bar R(z)) = (\alpha G_{\text{sc}}(z), G_{\text{sc}}(z), G_{\text{sc}}(z))$. In the process we obtain
$\bar q(\tau) \equiv - \int_{\mathbb{C}}\frac{dz}{2\pi i} \bar Q_\tau(z)$. 

\begin{proof}[Proof of formulas \eqref{eq:hatq0} and \eqref{eq:hatp0} in theorem \ref{th:risktrack}]
We use a change of variable $\hat Q_\tau(z) = e^{F(\tau)} \bar Q_\tau(z)$ and $\hat P_\tau(z) = e^{2 F(\tau)} \bar P_\tau(z)$ with
$F(\tau) = \int_0^\tau ds \biggl(\bar q^2(s) + \frac{1}{\sqrt \lambda} \bar p_1(s) \biggr)$.
%\begin{align}\label{eq:F}
%F(\tau) = \int_0^\tau \biggl(q_0^2(s) + \frac{1}{\sqrt \lambda} p_1(s) \biggr) \dd s
%\end{align}
Similarly, we define also $\hat q(\tau) = e^{F(\tau)} \bar q(\tau), \hat p(\tau) = e^{2F(\tau)}$. We have 
$\bar q(\tau) = \hat q(\tau)/\sqrt{\hat p(\tau)}$,
and therefore in order to determine the overlap it suffices to determine $\hat q(\tau)$ and $\hat p(\tau)$.
With the change of variables equations \eqref{eq:integrodiff} become
 \begin{align}\label{eq:chgvar}
        \begin{cases}
            \frac{d}{d\tau}\hat Q_\tau(z) = \hat q(\tau)\left( \bar R(z) + \frac{1}{\sqrt \lambda} \right) + \frac{z}{\sqrt \lambda} \hat Q_\tau(z) 
            \\
            \frac12 \frac{d}{d\tau} \hat P_\tau(z) =
            \hat q(\tau) \hat Q_\tau(z) + \frac{1}{\sqrt \lambda} \hat p(\tau) + \frac{z}{\sqrt \lambda} \hat P_\tau(z)  
         \end{cases}
\end{align}

We analyze these equations in the Laplace domain. Recall the Laplace transformation $\mathcal Lf(p) = \int_0^{+\infty} d\tau e^{-p\tau} f(\tau)$,   $\Re p > a\in \mathbb{R}_+$, which is well defined as long as $\vert f(\tau)\vert \leq e^{a\tau}$. 
All functions involved below in Laplace transforms satisfy this requirement for some $a \mathbb{R}_+$ large enough independent of $n$. It will often be convenient to use the notations $\mathcal L(f(t))(p) = \int_0^{+\infty} d\tau e^{-p\tau} f(\tau)$, 
$\mathcal L Q_p(z) = \int_0^{+\infty} d\tau e^{-p\tau} Q_\tau(z)$, $\mathcal L P_p(z) = \int_0^{+\infty} d\tau e^{-p\tau} P_\tau(z)$.

{\it A) Derivation of \eqref{eq:hatq0} for $\hat q(\tau)$}.
Taking the Laplace transform of the first equation in \eqref{eq:chgvar}
\begin{equation}
p \mathcal L \hat Q_p(z) - \hat Q_0(z) = \mathcal L \hat q(p) \bigl( \bar R(z) + \frac{1}{\sqrt \lambda} \bigr) + \frac{z}{\sqrt \lambda} \mathcal L \hat Q_p(z)
\end{equation}
Notice that $\hat Q_0(z) = e^{F(0)} \bar Q_0(z) = \alpha G_{\text{sc}}(z)$ and $\bar R(z)=G_{\text{sc}}(z)$, and hence we can re-arrange the terms,
\begin{equation}
\mathcal L \hat Q_p(z) = 
\alpha \frac{\sqrt \lambda G_{\text{sc}}(z)}{ p \sqrt \lambda - z}
+ \mathcal L \hat q(p) \frac{ \sqrt \lambda G_{\text{sc}}(z) + 1 }{p\sqrt \lambda - z}.
 \label{eq:qtlaplace}
\end{equation}
Now, assuming $\Re p> \frac{2+\delta}{\sqrt{\lambda}}$ (recall $\Re p > 0$), we can choose a sufficiently small contour $\Gamma$ around the pole $z = p\sqrt\lambda$ which does not traverse the interval $I_\delta$. Then, since $G_{\text{sc}}(z)$ is holomorphic in the interior of $\Gamma$, we get   
%such that $\forall z \in \Gamma, 2<|z|<p \sqrt \lambda$, and by inversion of laplace transformation integration and contour integration we get: % why ?
%\begin{equation}
%        \frac{1}{\sqrt \lambda} \mathcal L \hat q_0(p) =
%        \alpha \left[ \frac{-1}{2 i \pi} \oint \frac{ G_{\text{sc}}(z)}{ p \sqrt \lambda - z} \mathrm dz \right]
%        + \mathcal L \hat q_0(p) \left[ \frac{-1}{2 i \pi} \oint \frac{ G_{\text{sc}}(z)}{ p \sqrt \lambda - z} \mathrm dz \right]
%    \end{equation}
\begin{align}
0 & =
\alpha \oint_\Gamma \frac{dz}{2\pi i}\frac{ \sqrt{\lambda}G_{\text{sc}}(z)}{ p \sqrt \lambda - z} 
+ \mathcal L \hat q(p) \oint_\Gamma \frac{dz}{2\pi i} \frac{ \sqrt{\lambda}G_{\text{sc}}(z) +1}{ p \sqrt \lambda - z}
\nonumber \\ &
= \alpha \sqrt{\lambda} G_{\text{sc}}(p\sqrt{\lambda}) + \mathcal L \hat q(p) (\sqrt{\lambda} G_{\text{sc}}(p\sqrt{\lambda}) +1)
\end{align}
Thus we find:
\begin{align}
\mathcal L \hat q(p) = - \frac{\alpha G_{\text{sc}}(p\sqrt{\lambda})}{\frac{1}{\sqrt \lambda} + G_{\text{sc}}(p\sqrt{\lambda})}
= 
\alpha \frac{1 + \frac{1}{\sqrt \lambda} G_{\text{sc}}(p\sqrt{\lambda})}{p - (1 + \frac{1}{\lambda})}
\label{eq:Lq0}
\end{align}
where the last equality can be checked from the explicit expression \eqref{eq:Stieltjes} of $G_{\text{sc}}(z)$.
%
%
%    Using the semi-circle Stieltjes property: $ G_{\text{sc}}(z_0)  + \frac{1}{G_{\text{sc}}(z_0) } + z_0 = 0$ which can be rewritten $G_{\text{sc}}(z_0)^2 + z_0 G_{\text{sc}}(z_0) + 1= 0$ and in the canonical form:
%    $ \left( G_{\text{sc}}(z_0) + \frac{z_0}{2} \right)^2 = \left(\frac{z_0}{2}\right)^2 - 1$
%    and thus: 
%    \begin{equation}
%        \frac{1}{\frac{1}{\sqrt \lambda} + G_{\text{sc}}(z_0)} =
%        \frac{
%            \left(\frac{1}{\sqrt \lambda} - \frac{z_0}{2}\right) - \left(\frac{z_0}{2} + G_{\text{sc}}(z_0)\right)
%        }{ 
%            \left(\frac{1}{\sqrt \lambda} - \frac{z_0}{2}\right)^2 - \left(\frac{z_0}{2} + G_{\text{sc}}(z_0)\right)^2
%        }
%        =
%        \frac{
%            \frac{1}{\sqrt \lambda} - z_0 - G_{\text{sc}}(z_0)
%        }{ 
%            \frac{1}{\lambda} - \frac{z_0}{\sqrt \lambda} +1
%        }
%    \end{equation}
%    Therefore, multiplying by $- \alpha G_{\text{sc}}(z_0) $ we conclude:
%    \begin{equation}
%        \mathcal L \hat q_0(p) =
%        -\alpha \frac{
%            1 + \frac{1}{\sqrt \lambda} G_{\text{sc}}(z_0) 
%        }{
%            p - (1 + \frac{1}{\lambda})
%        }
%
%
It remains to invert this equation in the time domain. To do so 
we first notice that 
\begin{align}\label{eq:remark}
G_{\text{sc}}(p\sqrt{\lambda})
= 
-  \frac{1}{\sqrt\lambda} \int_{-2}^2 ds \mu_{\rm sc}(s)  \int_0^{+\infty} d\tau \, e^{(\frac{s}{\sqrt{\lambda}} -p) \tau}
= 
-  \frac{1}{\sqrt\lambda}\int_0^{+\infty} d\tau \, e^{-p\tau} M_\lambda(\tau)
\end{align}
where we recall that $M_\lambda(\tau)$ is the scaled moment generating function of the semi-circle law \eqref{eq:MGF}.
The interchange of integrals in the third equality is justified by Fubini. Using 
$\mathcal L (e^{(1+ \frac{1}{\lambda}) \tau})(p) = (p - (1 + \frac{1}{\lambda}))^{-1}$, equation \eqref{eq:Lq0} becomes 
%
%\begin{align}
$
\mathcal L \hat q(p)  = \alpha \mathcal L (e^{(1 + \frac{1}{\lambda}) t}) (p) 
- \frac{\alpha}{\lambda} \mathcal L (e^{(1 + \frac{1}{\lambda}) t}) \mathcal L M_\lambda(p)
$.
%\end{align}
%
This is easily transformed back in the time-domain using standard properties of the Laplace transform to get \eqref{eq:hatq0}.

{\it B) A useful identity.} For the derivation of $\hat p(\tau)$ we will need the following identity derived in Appendix \ref{app:technical}
\begin{align}\label{eq:claimaboutQ}
- \oint_{\mathcal{C}} \frac{dz}{2\pi i} \hat Q_u(z) e^{\frac{z(\tau -u)}{\sqrt\lambda}} = 
\alpha M_\lambda(2\tau - u)
            + \int_0^u ds \hat q(s) M_\lambda(2\tau-u-s)
\end{align}
where we recall $\mathcal{C}$ is the circle with center the origin and radius $\rho > 2+\delta$.

{\it C) Derivation of $\hat p(\tau)$}.
Taking the Laplace transform of the second equation in \eqref{eq:chgvar} we find
 \begin{equation}
 \frac12 (p \mathcal L \hat P_p(z) - \hat P_0(z)) = 
 \mathcal L(\hat q(\tau) \hat Q_\tau(z)) (p)
 + \frac{1}{\sqrt \lambda} \mathcal L \hat p(p) + \frac{z}{\sqrt \lambda} \mathcal L \hat P_p(z) 
 \end{equation}
and using $\hat P_0(z) = e^{F(0)} \bar P_0(z) = G_{\text{sc}}(z)$ we can rearrange the terms to get
\begin{align}\label{eq:lpz}
\mathcal L \hat P_p(z) = \frac{1}{p - \frac{z\sqrt\lambda}{2}}
\biggl(
G_{\text{sc}}(z) 
+ 2\sqrt\lambda \mathcal L(\hat q(\tau) \hat Q_\tau(z)) (p)
+ \frac{2}{\sqrt\lambda} \mathcal L \hat p(p)
\biggr).
\end{align}
%Similarly than before we consider $p\in \mathbb{C}\setminus [-2, 2]$ and $\Re p >0$ and make a contour integral along a sufficiently small circle around $z = p\sqrt\lambda /2$ such that it does not traverse the interval $[-2, 2]$. We find using \eqref{eq:complicated}
%\begin{align}
%\mathcal L \hat p_0(p) = - \frac{\sqrt\lambda}{2} G_{\text{sc}}(\frac{p\sqrt \lambda}{2}) - \sqrt\lambda\oint_{\Gamma^\prime} \frac{dz}{2\pi i} \frac{
%\mathcal L(\hat q_0(\tau) \hat Q_\tau(z) (p)}{\frac{p\sqrt\lambda}{2} -z}
%\end{align}
%and using \eqref{eq:remark}
%\begin{align}
%\mathcal L \hat p_0(p) = - \mathcal{L} I_{\frac{\lambda}{4}}(p) - \sqrt\lambda \mathcal L(\hat q_0(\tau) \hat Q_\tau(\frac{p\sqrt \lambda}{2})) (p)
%\end{align}
Then using $(p-\frac{2z}{\sqrt \lambda})^{-1} = \mathcal L (e^{\frac{2zt}{\sqrt \lambda}})(p)$ and 
\begin{equation}
2 \mathcal L (e^{\frac{2zt}{\sqrt \lambda}}) \mathcal L(\hat q(t) \hat Q_t(z)) 
= \mathcal L \bigl( 2 \int_{0}^t \hat q(u) \hat Q_u(z) e^{\frac{2z(t-u)}{\sqrt \lambda}} du \bigr),
\end{equation}
and replacing in \eqref{eq:lpz} we get
%    Notice that $\hat Q_t(z)$ is completely found through equation Eq. \ref{eq:qtlaplace}:
%    \begin{equation}
%        \mathcal L \hat Q_p(z) = \alpha G_{\text{sc}}(z) \mathcal L (e^{\frac{zt}{\sqrt\lambda}})
%        + \mathcal L \hat q_0(p) \mathcal L (e^{\frac{zt}{\sqrt\lambda}}) \left(G_{\text{sc}}(z) + \frac{1}{\sqrt \lambda}\right)
%    \end{equation}
%    Which translates for $u\in\mathbb R_+$ into:
%    \begin{equation}
%        \hat Q_u(z) = \alpha G_{\text{sc}}(z) e^{\frac{zu}{\sqrt\lambda}}
%        + \left(G_{\text{sc}}(z) + \frac{1}{\sqrt \lambda}\right) \int_0^u \hat q_0(s)  e^{\frac{z(u-s)}{\sqrt\lambda}}  \dd s
%    \end{equation}
\begin{align}\label{eq:weget}
\mathcal L \hat P_p(z) = \mathcal{L}\bigl(e^{\frac{2z\tau}{\sqrt \lambda}}\bigr)(p) G_{\text{sc}}(z) + 2 \mathcal{L}\biggl(\int_{0}^\tau \hat q(u) \hat Q_u(z) e^{\frac{2z(\tau-u)}{\sqrt \lambda}} du   \biggr)(p) + \frac{2}{\sqrt \lambda}\mathcal L (e^{\frac{2z\tau}{\sqrt \lambda}})(p) \mathcal L \hat p(p).
\end{align}
Now we take $\Re p > 4 / \sqrt{\lambda}$ and choose the contour $\mathcal{C}$, encircling the interval $I_\delta$,  but such that  it does not encircle the point $z=\frac12 p \sqrt{\lambda}$, and integrate each term of \eqref{eq:weget} along this contour. First note that the contribution of the last term vanishes since $\mathcal{L} (e^{\frac{2z\tau}{\sqrt \lambda}})(p) = (p - \frac{2z}{\sqrt\lambda})^{-1}$ and the pole 
$z=\frac12 p \sqrt \lambda$ lies in the exterior of $\mathcal{C}$. Then there remains
\begin{align}
\oint_{\mathcal{C}} \frac{dz}{2\pi i} \mathcal{L}\hat P_p(z) 
= 
\oint_{\mathcal{C}} \frac{dz}{2\pi i} \mathcal{L}(e^{\frac{2z\tau}{\sqrt \lambda}})(p) G_{\text{sc}}(z) 
+ 
2\oint_{\mathcal{C}} \frac{dz}{2\pi i} \mathcal{L}\biggl(
\int_{0}^\tau \hat q(u) \hat Q_u(z) e^{\frac{2z(\tau-u)}{\sqrt \lambda}} du\biggr)(p) .
\label{eq:preparation}
\end{align}
For the left hand side we have
\begin{align}
\oint_{\mathcal{C}} \frac{dz}{2\pi i} \int_0^{+\infty} d\tau e^{-p\tau} \hat P_\tau(z)
= \int_0^{+\infty} d\tau e^{-p\tau} \oint_{\mathcal{C}} \frac{dz}{2\pi i}  \hat P_\tau(z) = - \int_0^{+\infty} d\tau e^{-p\tau}\hat p(\tau)
\label{eq:start}
\end{align}
where the first equality follows from Fubini and the second by functional calculus \cite{schwartz1958linear}.
For the first term on the right hand side of \eqref{eq:preparation} we find (see Appendix \ref{app:technical} for details)
\begin{align}
\oint_{\mathcal{C}} \frac{dz}{2\pi i} G_{\text{sc}}(z) \int_0^{+\infty} d\tau e^{-p\tau} e^{\frac{z\tau}{\sqrt \lambda}}
=
- \int_0^{+\infty} d\tau e^{-p\tau}M_\lambda(\tau) .
\label{eq:almost}
\end{align}
Finally it remains to treat the last contour integral in \eqref{eq:preparation}.
Using again Fubini and \eqref{eq:claimaboutQ} we find
\begin{align}
& \oint_{\mathcal{C}} \frac{dz}{2\pi i} \int_0^{+\infty} d\tau e^{-p\tau}
\int_{0}^\tau \hat q(u) \hat Q_u(z) e^{\frac{2z(\tau-u)}{\sqrt \lambda}} du
 =
\int_0^{+\infty} d\tau e^{-p\tau} \int_{0}^\tau \hat q(u)  \oint_{\mathcal{C}} \frac{dz}{2\pi i} \hat Q_u(z) e^{\frac{2z(\tau-u)}{\sqrt \lambda}} du
\nonumber \\ &
= 
- \int_0^{+\infty} d\tau e^{-p\tau} \int_{0}^\tau du \hat q(u) \biggl[
\alpha M_\lambda(2\tau - u)
            + \int_0^u ds \hat q(s) M_\lambda(2\tau-u-s)\biggr]
\nonumber \\ &
=
- \int_0^{+\infty} d\tau e^{-p\tau}\biggl[ \alpha \int_{0}^\tau du \hat q(u) M_\lambda(2\tau - u) + 
 \frac{1}{2} \int_0^\tau \int_0^\tau du ds \hat q(s) M_\lambda(2\tau-u-s)\biggr]
\label{eq:atlast}
\end{align}
Putting together \eqref{eq:preparation}, \eqref{eq:start}, \eqref{eq:almost}, \eqref{eq:atlast} we obtain \eqref{eq:hatp0} in the Laplace domain. Going back to the time domain we obtain \eqref{eq:hatp0}.
\end{proof}

\section{Conclusion and future work}

Tracking gradient descent dynamics and their variants for different scores and loss functions can be used to provide meaningful insights on a learning algorithm and for example, help monitor its progress and avoid over-fitting. 
%As computational capabilities increase with distributed systems allowing for bigger datasets and larger systems to be treated, the dynamic is critical as computational cost needs to be accounted.
We have seen in this work that for 
the rank-one matrix recovery problem in the regime of large dimensions, probabilistic concentrations naturally occur that can be captured by the local semi-circle laws in random matrix theory obtained in the last decade. In particular, suitable generating functions constructed out of the resolvent of the noise matrix concentrate around the solutions of a set of deterministic integro-differential equations. We have been able to completely solve these equations thereby tracking the dynamics for all times. It is also observed that the analytical solution
%These concentrations emerge through the quantities describing the gradient descent dynamic of the problem. As in statistical physics and thermodynamics, this lead us to approximate the system with its limiting behavior. We have seen that it provides a set of integro-differential equations and solved it to find the limiting solution of the dynamic. We have shown that this solution is the expected solution at all time as the dimension of the system is increasing, and that in general circumstances, it 
provides a good approximation for the expected behavior of the learning algorithm, even for dimensions as low as $n<100$.
We expect that the method and integro-differential equations derived here can be generalized to different models. For instance, one may be able to apply it to certain neural-network architectures, and in particular the random feature models. This would allow to better understand the dynamical emergence of interesting behaviors such as the double descent phenomenon.

%\input{main-aknowledgment}

% Acknowledgments---Will not appear in anonymized version
%\acks{}
\bibliographystyle{plain}

\bibliography{refs}

\newpage
\appendix

\section{Analysis of the cost}\label{cost-analysis}

\begin{proof}[Proof of theorem \ref{thmrisktracktrue}]
Expanding the Frobenius norm in the cost and using $\Vert \theta_\tau\Vert^2 = \Vert \theta^*\Vert^2 = n$ we find
\begin{align}
\mathcal{H}(\theta_\tau) & = \frac{1}{2 n^2}\bigl\{-2 {\rm Tr} Y \theta_\tau \theta_{\tau}^T + {\rm Tr} (\theta_\tau \theta_{\tau}^T \theta_{\tau}\theta_{\tau}^T)\bigr\} 
-
\frac{1}{2 n^2}\bigl\{-2 {\rm Tr} Y \theta^* \theta^{*T} + {\rm Tr} (\theta^* \theta^{*T} \theta^{*}\theta^{*T})\bigr\} 
\nonumber \\ &
=
\frac{1}{n^2} \langle \theta^*, Y \theta^* \rangle   - \frac{1}{n^2} \langle \theta_\tau, Y \theta_\tau \rangle .
\end{align}
Using that $Y = \theta^* \theta^{*T} + \frac{n}{\sqrt \lambda} H$ (recall $H = \frac{1}{\sqrt n} \xi$) we get
\begin{align}
\mathcal H(\theta_\tau) & = \bigl(1+ \frac{1}{\sqrt\lambda} \frac{\langle \theta^*, H \theta^*\rangle}{n}\bigr) 
- 
\bigl( \frac{\langle \theta_\tau, \theta^* \rangle^2}{n^2} + 
 \frac{1}{\sqrt \lambda} \frac{\langle \theta_\tau, H \theta_t\rangle}{n}
 \bigr) 
  \nonumber \\ &
  =
  \bigl(1+ \frac{1}{\sqrt\lambda} \frac{\langle \theta^*, H \theta^*\rangle}{n}\bigr)
    - \bigl( q(\tau)^2 + 
  \frac{p_1(\tau)}{\sqrt \lambda}
  \bigr).
 \label{eq:costqp}
\end{align}
By the law of large numbers $\frac{\langle \theta^*, H \theta^*\rangle}{n} \overset{\mathbb{P}}{\underset{n \to\infty}{\longrightarrow}} 0$
and since $q(\tau) \overset{\mathbb{P}}{\underset{n \to\infty}{\longrightarrow}} \bar q(\tau)$ and 
$p_1(\tau) \overset{\mathbb{P}}{\underset{n \to\infty}{\longrightarrow}} \bar p_1(\tau)$ we have 
\begin{align}\label{eq:Rqp}
\mathcal{H}(\theta_\tau) \overset{\mathbb{P}}{\underset{n \to\infty}{\longrightarrow}} 1 - \biggl( \bar q(\tau)^2 + 
  \frac{\bar p_1(\tau)}{\sqrt \lambda}
  \biggr) .
\end{align}
Now it remains to recall the definition of $F(\tau)$ and $\hat p_0(\tau) = e^{2F(\tau)}$, to see that
\begin{align}\label{eq:qpF}
\bar q(\tau)^2 + \frac{\bar p_1(\tau)}{\sqrt \lambda}
=
\frac{dF(\tau)}{d\tau} = \frac{1}{2}\frac{d}{d\tau} \ln \hat p(\tau) .
\end{align}
The result of the theorem follows from \eqref{eq:Rqp} and \eqref{eq:qpF}.
\end{proof}
    
\section{Proof of propositions \ref{local-semi-circle-version} and \ref{uniform-local-semi-circle-version}}\label{app:initial-conditions}

The proof is based the following {\it local} semi-circle law (theorem 2.12 in \cite{alex2014}):

\begin{theorem}[isotropic local semi-circle law \cite{alex2014}]\label{th:ilsl}
For any $\omega \in (0,1)$ consider the following domain in the upper half-plane
$$
S(\omega, n) = \left\{ z \in \mathbb{C} \mid |\Re(z)| \leq \frac{1}{\omega} , \frac{1}{n^{1-\omega}} \leq \Im(z) \leq \frac{1}{\omega} \right\}.
$$
 Then for all $\delta, D>0$, there exists $n_0 \in \mathbb N$ such that for all $n > n_0$, and any unit vectors $u, v \in \mathbb S_n(1)$:
    \begin{equation}
        \sup_{z \in S(\omega, n)} \mathbb{P}\biggl( 
            \vert \langle u, \mathcal R(z) v\rangle - \langle u, v\rangle G_{\text{sc}}(z) \vert
            > n^\delta \biggl[ \sqrt \frac{\Im G_{\text{sc}}(z)}{n \Im z} + \frac{1}{n \Im z}  \biggr]
        \biggr) < \frac{1}{n^D}
    \end{equation}
where $\mathbb{P}$ is the probability law of the GOE.
\end{theorem}

\begin{proof}[\it Proof of proposition \ref{local-semi-circle-version}]
First we note that for $\Im z \neq 0$ since $\lim_{n\to +\infty}\mathbb{P}(\mathcal{S}_\delta^n) =1$ we have 
\begin{align}
\lim_{n\to +\infty}\mathbb{P} \bigl(
            \vert \langle u, \mathcal R(z) v\rangle - \langle u, v\rangle G_{\text{sc}}(z)\vert > \epsilon \bigr)
            =
            \lim_{n\to +\infty}\mathbb{P}_\delta \bigl(
            \vert \langle u, \mathcal R(z) v\rangle - \langle u, v\rangle G_{\text{sc}}(z)\vert > \epsilon \bigr).
\end{align}
We consider fist the cases $\Im z$ strictly positive, negative, and then give the extra argument needed for 
$\Im z =0$.  

First we take $\Im z > 0$. We can find $n_1 \in \mathbb N, \omega \in (0,1)$ such that $z\in S(\omega, n_1)$ and henceforth, for all $n \geq n_1$, $z \in S(\omega, n)$.
Taking $\delta = \frac14, D=1$ and applying theorem \ref{th:ilsl} yields the existence of $n_0$ such that for all $n \geq \max(n_0, n_1)$:
    \begin{equation}
        \mathbb{P}\biggl( 
                \vert \langle u, \mathcal R(z) v\rangle - \langle u, v\rangle G_{\text{sc}}(z)\vert
                > n^\frac14 \biggl[ \sqrt \frac{\Im G_{\text{sc}}(z)}{n \Im z} + \frac{1}{n \Im z}  \biggr]
            \biggr) < \frac{1}{n} .
    \end{equation}
Set $l(n,z) = n^\frac14 \biggl[ \sqrt \frac{\Im G_{\text{sc}}(z)}{n \Im z} + \frac{1}{n \Im z} \biggr]$. Since $\lim_{n\to \infty} l(n,z) = 0$, we can find $n_2$ such that for all $n \geq n_2$ we have $l(n,z) < \epsilon$.
Thus for all $n \geq \max(n_0,n_1,n_2)$ we have the set inclusion in the GOE
    \begin{equation}
        \left\{H :
            \vert \langle u, \mathcal R(z) v\rangle - \langle u, v\rangle G_{\text{sc}}(z) \vert > \epsilon
        \right\} \subset \left\{H :
            \vert \langle u, \mathcal R(z) v\rangle - \langle u, v\rangle G_{\text{sc}}(z)\vert | > l(n, z)
        \right\}
    \end{equation}
    and therefore
    \begin{equation}
        \mathbb{P} \bigl(
            \vert \langle u, \mathcal R(z) v\rangle - \langle u, v\rangle G_{\text{sc}}(z)\vert > \epsilon \bigr) < \frac{1}{n} .
    \end{equation}
 Applying this inequality to a deterministic sequence $(u^{(n)}, v^{(n)})$ on the unit sphere and taking the limit $n \to \infty$ concludes the proof for 
  $\Im z >0$.
  
To deal with $\Im z < 0$ it suffices to remark that $\vert \langle u, \mathcal R(z) v\rangle - \langle u, v\rangle G_{\text{sc}}(z)\vert
= \vert \langle u, \mathcal R(\bar z) v\rangle - \langle u, v\rangle G_{\text{sc}}(\bar z)\vert$. Alternatively one could use a version of theorem \ref{th:ilsl} for the lower half-plane.  

Consider now $z = x$ with $x\in\mathbb{R} \setminus I_\delta$ and $H\in \mathcal{S}_\delta^n$. Take a complex number 
$x + i y$, $0 < y \leq \frac{\epsilon}{2}\vert x - (2 + \delta)\vert^2$. From the mean value theorem we have 
\begin{align}
& \vert (\langle u, \mathcal R(x) v\rangle - \langle u, v\rangle G_{\text{sc}}(x))
-
(\langle u, \mathcal R(x+iy) v\rangle - \langle u, v\rangle G_{\text{sc}}(x+iy_k))\vert 
\nonumber \\ &
\leq 
\vert  y\vert  \sup_{y > 0} \vert \frac{d}{dy} \langle u, \mathcal{R}(x+iy) v\rangle\vert .
\label{eq:annex}
\end{align}
Since for $H\in \mathcal{S}_\delta^n$
\begin{align}
\vert \frac{d}{dy} \langle u, \mathcal{R}(x+iy) v\rangle\vert = \vert\langle u, (x+ iy - H)^{-2} v\rangle\vert \leq \frac{1}{(x - (2+\delta))^2 + y^2}
\end{align}
we deduce from \eqref{eq:annex} and the triangle inequality
\begin{align}
\vert \langle u, \mathcal R(x) v\rangle - \langle u, v\rangle G_{\text{sc}}(x)\vert 
& \leq 
\vert \langle u, \mathcal R(x+iy) v\rangle - \langle u, v\rangle G_{\text{sc}}(x+iy_k)\vert 
+
\frac{y}{\vert x - (2 + \delta)\vert^2}
\nonumber \\ &
\leq 
\vert \langle u, \mathcal R(x+iy) v\rangle - \langle u, v\rangle G_{\text{sc}}(x+iy_k)\vert 
+ \frac{\epsilon}{2} .
\end{align}
Thus for realizations $H\in \mathcal{S}_\delta^n$, the event $\vert \langle u, \mathcal R(x) v\rangle - \langle u, v\rangle G_{\text{sc}}(x)\vert > \epsilon$ implies the event $\vert \langle u, \mathcal R(x+iy) v\rangle - \langle u, v\rangle G_{\text{sc}}(x+iy)\vert \geq \frac{\epsilon}{2}$ for any $0< y \leq \frac{\epsilon}{2}\vert x - (2+\delta)\vert^2$. In other words
\begin{align}
\mathbb{P}_\delta\bigl(\vert \langle u, \mathcal R(x) v\rangle - \langle u, v\rangle G_{\text{sc}}(x)\vert > \epsilon \bigr)
\leq 
\mathbb{P}_\delta\bigl(\vert \langle u, \mathcal R(x+iy) v\rangle - \langle u, v\rangle G_{\text{sc}}(x+iy)\vert \geq \frac{\epsilon}{2} \bigr).
\end{align}
By the previous results for $\Im z > 0$ we conclude that these probabilities tend to zero as $n\to +\infty$. 

\end{proof}

\begin{proof}[\it Proof of proposition \ref{uniform-local-semi-circle-version}]
The proof uses a discretization argument together with the union bound. Consider the discrete set of $N$ points on the contour $\mathcal{C}$, $z_k = \rho e^{i \theta_k}$, $\theta_k= \frac{2\pi k}{N}$, $k=0, \dots, N-1$. First, Observe that from the union bound
\begin{align}
\mathbb{P}\bigl(\max_{k=0, \cdots, N} \vert \langle u^{(n)}, \mathcal R(z_k) v^{(n)} \rangle & - \langle u^{(n)}, v^{(n)}\rangle G_{\text{sc}}(z_k) \vert > \epsilon \bigr) 
\nonumber \\ &
\leq 
\sum_{k=0}^N \mathbb{P}\bigl(\vert \langle u^{(n)}, \mathcal R(z_k) v^{(n)} \rangle - \langle u^{(n)}, v^{(n)}\rangle G_{\text{sc}}(z_k) \vert > \epsilon \bigr)
\end{align}
thus from proposition \eqref{local-semi-circle-version} 
\begin{align}\label{eq:prob1}
\lim_{n\to +\infty} \mathbb{P}\bigl(\max_{k=0, \cdots, N} \vert \langle u^{(n)}, \mathcal R(z_k) v^{(n)} \rangle & - \langle u^{(n)}, v^{(n)}\rangle G_{\text{sc}}(z_k) \vert > \epsilon \bigr) = 0
\end{align}
Second, for any $z = \rho e^{i\theta}\in \mathcal{C}$ there exist a $\theta_k$ such that $\vert \theta - \theta_k\vert \leq \frac{1}{N}$. 
Applying the triangle inequality $\vert b\vert \leq \vert a\vert + \vert b -a\vert$ for 
$a = \langle u^{(n)}, \mathcal R(z) v^{(n)} \rangle - \langle u^{(n)}, G_{\rm sc}(z) v^{(n)} \rangle$ and 
$b = \langle u^{(n)}, \mathcal R(z_k) v^{(n)} \rangle - \langle u^{(n)}, G_{\rm sc}(z_k) v^{(n)} \rangle$, and the mean value theorem, 
we get
\begin{align}
\vert \langle u^{(n)}, \mathcal R(z_k) v^{(n)} \rangle - \langle u^{(n)}, G_{\rm sc}(z_k) v^{(n)} \rangle\vert
& \leq 
\vert \langle u^{(n)}, \mathcal R(z) v^{(n)} \rangle - \langle u^{(n)}, \mathcal R_H(z) v^{(n)} \rangle\vert
\nonumber \\ &
+
\frac{1}{N} \sup_{\theta\in [0, 2\pi]} \vert \frac{d}{d\theta} \langle u^{(n)}, \mathcal R(\rho e^{i\theta}) v^{(n)} \rangle\vert
\end{align}
We can take the supremum of the right hand side over 
$z\in\mathcal{C}$ and then the maximum of the right hand side over $k=0, \dots, N-1$ to deduce
\begin{align}
\max_{k=0, \dots, N}\vert \langle u^{(n)}, \mathcal R(z) v^{(n)} \rangle - \langle u^{(n)}, G_{\rm sc}(z) v^{(n)} \rangle\vert
& \leq 
\sup_{z\in\mathcal{C}} \vert \langle u^{(n)}, \mathcal R(z_k) v^{(n)} \rangle - \langle u^{(n)}, \mathcal R(z_k) v^{(n)} \rangle\vert
\nonumber \\ &
+
\frac{1}{N} \sup_{\theta\in [0, 2\pi]} \vert \frac{d}{d\theta} \langle u^{(n)}, \mathcal R(\rho e^{i\theta}) v^{(n)} \rangle\vert
\label{eq:triangle}
\end{align}
Since 
\begin{align}
\frac{d}{d\theta} \langle u^{(n)}, \mathcal R(\rho e^{i\theta}) v^{(n)} \rangle = i\rho e^{i\theta} \langle u^{(n)}, (\rho e^{i\theta} - H)^{-2} v^{(n)} \rangle
\end{align}
we deduce from Cauchy-Schwarz, that with probability tending to one as $n\to + \infty$
\begin{align}\label{eq:prob2}
\frac{1}{N} \sup_{\theta\in [0, 2\pi]} \vert \frac{d}{d\theta} \langle u^{(n)}, \mathcal R(\rho e^{i\theta}) v^{(n)} \rangle\vert
\leq 
\frac{\rho}{N (\rho - 2)^2}
\end{align}
Therefore taking $N > \frac{2\rho}{\epsilon (\rho -2)^2}$ we find from \eqref{eq:prob1}, \eqref{eq:prob2} and \eqref{eq:triangle}
\begin{align}
\lim_{n\to + \infty}
\mathbb{P}\bigl(\sup_{z\in\mathcal{C}} \vert \langle u^{(n)}, \mathcal R(z) v^{(n)} \rangle & - \langle u^{(n)}, v^{(n)}\rangle G_{\text{sc}}(z) \vert \geq  \frac{\epsilon}{2} \bigr)
= 0
\end{align}
for any $\epsilon > 0$. This concludes the proof.
\end{proof}

\section{Proof of proposition \ref{concentration-a-tout-temps-fini}}\label{app:proof-gronwall}

We assume the condition 
$H\in \mathcal{S}_\delta^n$ so that 
$\Vert \mathcal{R}(z)\Vert_{\rm op} \leq (\rho -2)^{-1}$ for
all $z\in \mathcal{C} = \{z\in \mathbb{C} \mid z = \rho e^{i\theta}, \theta\in [0, 2\pi]\}$ and $\rho > 2+\delta$. The condition is relaxed at the very end.

The proof of proposition \ref{concentration-a-tout-temps-fini} is based on a Gronwall type argument. 
As explained in section \ref{concentration} the difficulty here is that we have an integro-differential equation instead of a plain ordinary differential equation and the usual Lipshitz condition is not a priori satisfied. For this reason, given that $H\in \mathcal{S}_\delta^n$, we need preliminary bounds 
on $\sup_{z\in\mathcal{C}} \vert Q_\tau(z)\vert $, $\sup_{z\in\mathcal{C}} \vert P_\tau(z)\vert $,
$\sup_{z\in\mathcal{C}} \vert R(z)\vert$,  $\sup_{z\in\mathcal{C}} \vert \bar R(z)\vert$ and on 
$\sup_{z\in\mathcal{C}} \vert \bar Q_\tau(z)\vert$, $\sup_{z\in\mathcal{C}} \vert\bar P_\tau(z)\vert$, for $\tau\in [0, T]$. 
Here we do not seek the best possible bounds but rather we just need that all quantities are bounded (with high probability for the first three). 

For the first {\it four} quantities the bound easily follows from their definition \eqref{eq:genfunctions}.  By Cauchy-Schwartz we obtain
that $\sup_{z\in\mathcal{C}} \vert Q_\tau(z)\vert $, $\sup_{z\in\mathcal{C}} \vert P_\tau(z)\vert$ and $\sup_{z\in\mathcal{C}} \vert R(z)\vert$ are upper bounded by $(\rho -2)^{-1}$. For $\sup_{z\in\mathcal{C}} \vert\bar R(z)\vert$ we can use the integral representation to get the same (loose) bound.

The remaining {\it two} quantities are here defined through the solution of the integro-differential equation \eqref{eq:limitingdiff} which we take as a starting point to prove a bound. In section \ref{analysis-gradient-descent} we compute exactly the combination $\bar q(\tau)^2 + \frac{1}{\lambda}\bar p_1(\tau) \equiv \frac{1}{2}\ln \hat p(\tau)$ and find $\hat p(\tau)$ given by formula \eqref{eq:hatp0}. It can be checked that this is a continuous function for any compact time interval, so
$\sup_{\tau\in [0, T]} \vert \bar q(\tau)^2 + \frac{1}{\sqrt\lambda} \bar p_1(\tau)\vert \leq L_*(T) < +\infty$ for any $T>0$ (in fact one can even take $L_*$ independent of $T$ but we will not need this information). 
Then, integrating the first equation in \eqref{eq:limitingdiff} over $[0, \tau]$, using the triangle inequality, and then taking suprema, we deduce 
\begin{align}
\sup_{z\in \mathcal{C}}\vert \bar Q_\tau(z)\vert 
\leq &
\sup_{z\in \mathcal{C}}\vert \bar Q_0(z)\vert 
+
\bigl(\frac{\rho}{\rho -2} +  \frac{2\rho}{\sqrt \lambda}  + \rho^2 L_*(T) \bigr)\int_0^\tau ds\, \sup_{z\in \mathcal{C}}\vert \bar Q_s(z)\vert 
\end{align}
Iterating this inequality a standard calculation yields any $\tau\in [0, T]$
\begin{align}\label{eq:bound1}
\sup_{z\in \mathcal{C}}\vert \bar Q_\tau(z)\vert \leq \sup_{z\in \mathcal{C}}\vert \bar Q_0(z) \vert \, e^{T\bigl(\frac{\rho}{\rho -2} +  \frac{2\rho}{\sqrt \lambda}  + \rho^2 L_*(T) \bigr)} \leq  \frac{\alpha}{\rho -2} e^{T\bigl(\frac{\rho}{\rho -2} +  \frac{2\rho}{\sqrt \lambda}  + \rho^2 L_*(T) \bigr)}
\end{align}
where we used $\bar Q_0(z) = \alpha G_{\rm sc}(z)$, and for $\vert G_{\rm sc}(z)\vert \leq \frac{1}{\rho -2}$ for $z\in \mathcal{C}$. The definition of $\bar q(\tau)$ in terms of a contour integral implies immediately 
$\sup_{\tau\in [0, T]} \vert\bar q(\tau)\vert \leq L(T)$ where $L(T)$ is the right hand side of \eqref{eq:bound1} multiplied by $\rho$. 
Now, integrating the second equation in \eqref{eq:limitingdiff} over $[0, \tau]$, using the triangle inequality, and then taking suprema again, we deduce 
\begin{align}
\frac{1}{2}\sup_{z\in \mathcal{C}}\vert \bar P_\tau(z)\vert 
\leq &
\frac{1}{2}\sup_{z\in \mathcal{C}}\vert \bar P_0(z)\vert 
+
\frac{\alpha^2 \rho \tau}{(\rho -2)^2} e^{2T\bigl(\frac{\rho}{\rho -2} +  \frac{2\rho}{\sqrt \lambda}  + \rho^2 L_*(T) \bigr)}
+
\frac{\tau}{\sqrt\lambda}
\nonumber \\ &
+ 
\bigl(\frac{\rho}{\sqrt\lambda} + L_*(T)\bigr) \int_0^\tau \sup_{z\in\mathcal{C}} \vert \bar P_s(z)\vert
\end{align}
Again a standard calculation yields (using the initial condition $\bar P_0(z) = G_{\rm sc}(z)$)
\begin{align}\label{eq:bound2}
\sup_{z\in \mathcal{C}}\vert \bar P_\tau(z)\vert \leq 
\biggl(\frac{1}{\rho -2} + \frac{2\alpha^2 \rho T}{(\rho -2)^2} e^{2T\bigl(\frac{\rho}{\rho -2} +  \frac{2\rho}{\sqrt \lambda}  + \rho^2 L_*(T) \bigr)} + \frac{2\tau}{\sqrt\lambda}\biggr)
e^{T(\frac{\rho}{\sqrt\lambda} + L_*(T))}
\end{align}
Note that this implies the bound $\sup_{\tau\in [0, T]}\vert \bar p_1(\tau)\vert \leq L_1(T)$ where $L_1(T)$ is the right hand side of \eqref{eq:bound2} multiplied by $\rho^2$.

We now have all the elements to adapt a Gronwall type argument.

\begin{proof}[{\it Proof of proposition \ref{concentration-a-tout-temps-fini}}]
We start by deriving preliminary bounds 
We set $Q_\tau(z) - \bar Q_\tau(z) = \Delta^{Q}_\tau(z)$, $P_\tau(z) - \bar P_\tau(z) = \Delta^{P}_\tau(z)$, 
$R(z) - \bar R(z) = \Delta^{R}(z)$, $q(\tau) - \bar q(\tau) = \delta^{q}(\tau)$, 
$p_1(\tau) - \bar p_1(\tau) = \delta^{p_1}(\tau)$. Note for later use that all the $\sup_{z\in\mathcal{C}}\vert \cdot \vert$ of these differences are bounded by some finite positive constant depending only on $\rho, \alpha, \lambda, T$. Taking the difference of \eqref{eq:limitingdiff} and \eqref{eq:integrodiff} we find after a bit of algebra
\begin{align}
\frac{d}{d\tau} \Delta^{Q}_\tau(z) 
 = & 
 \, \delta^{q}(\tau) \Delta^{R}(z) + \delta^{q}(\tau) \bar R(z) + \bar q(\tau)
\Delta^R(z) 
+ \frac{1}{\sqrt{\lambda}}\bigl(z \Delta^{Q}_\tau(z) + \delta^{q}(\tau)\bigr)
\nonumber \\ &
-
(q(\tau) + \bar q(\tau))\delta^{q}(\tau) \Delta^{Q}_\tau(z) - (q(\tau) + \bar q(\tau))\delta^{q}(\tau) \bar Q_\tau(z) 
- 
\bar q(\tau)^2\Delta^Q_\tau(z) 
\nonumber \\ &
- \frac{1}{\sqrt\lambda}\bigl(\delta^{p_1}(\tau) \Delta^{Q}_\tau(z) - \delta^{p_1}(\tau) \bar Q_\tau(z) - 
\bar p_1(\tau)\Delta^Q_\tau(z) \bigr)
\end{align}
and 
\begin{align}
\frac{d}{d\tau} \Delta^{P}_\tau(z) 
 = & 
 \, \delta^{q}(\tau) \Delta^{Q}_\tau(z) + \delta^{q}(\tau) \bar Q_\tau(z) + \bar q(\tau)
\Delta^Q_\tau(z) 
+ \frac{1}{\sqrt{\lambda}}z \Delta^{P}_\tau(z)
\nonumber \\ &
-
(q(\tau) + \bar q(\tau))\delta^{q}(\tau) \Delta^{P}_\tau(z) - (q(\tau) + \bar q(\tau))\delta^{q}(\tau) \bar P_\tau(z) 
- 
\bar q(\tau)^2\Delta^P_\tau(z) 
\nonumber \\ &
- \frac{1}{\sqrt\lambda}\bigl(\delta^{p_1}(\tau) \Delta^{P}_\tau(z) - \delta^{p_1}(\tau) \bar P_\tau(z) - 
\bar p_1(\tau)\Delta^P_\tau(z) \bigr)
\end{align}
After
integrating the above equations over the interval $[0, \tau]$, using the triangle inequality, and the inequalities
$\vert\delta^{q}(\tau)\vert \leq \rho \sup_{z\in \mathcal{C}} \vert\Delta^{Q}_\tau(z)\vert$, 
$\vert\delta^{p_1}(\tau)\vert \leq \rho^2 \sup_{z\in \mathcal{C}} \vert\Delta^{P}_\tau(z)\vert$, 
$\vert q(\tau)\vert \leq 1$, 
$\sup_{\tau\in [0, T]} \vert \bar q(\tau)\vert < L(T)$, 
$\sup_{\tau\in [0, T]} \vert \bar p_1(\tau)\vert < L_1(T)$, we deduce (with $L = max(L(T), L_1(T)$)
\begin{align}
 \sup_{z\in \mathcal{C}} & \vert \Delta^{Q}_\tau(z) \vert
 \leq  
 \sup_{z\in \mathcal{C}}\vert \Delta^{Q}_0(z) \vert
 +
 \rho \sup_{z\in\mathcal{C}}\vert \Delta^{R}(z)\vert \int_0^\tau ds\, \sup_{z\in \mathcal{C}} \vert\Delta^{Q}_s(z)\vert
 + 
 \rho \sup_{z\in \mathcal{C}}\vert \bar R(z)\vert \int_0^\tau ds\,\sup_{z\in \mathcal{C}} \vert\Delta^{Q}_s(z)\vert 
 \nonumber \\ &
 + 
 L\tau \sup_{z\in\mathcal{C}}\vert \Delta^R(z) \vert
+ 
\frac{2\rho}{\sqrt{\lambda}} \int_0^\tau ds\, \sup_{z\in\mathcal{C}}\vert \Delta^{Q}_s(z)\vert
+
(1 + L)\rho \int_0^\tau ds\, (\sup_{z\in \mathcal{C}}\vert \Delta^{Q}_s(z)\vert)^2 
\nonumber \\ &
+ 
(1 + L)\rho \int_0^\tau ds\, (\sup_{z\in \mathcal{C}}\vert \Delta^{Q}_s(z)\vert)^2 \sup_{z\in\mathcal{C}}\vert \bar Q_s(z)\vert 
+
L^2 \int_0^\tau ds\, \sup_{z\in\mathcal{C}}\vert \Delta^Q_s(z) \vert
\nonumber \\ &
+
\frac{\rho^2}{\sqrt\lambda}\int_0^\tau ds\, \sup_{z\in\mathcal{C}} \vert \Delta^P_s(z)\vert \sup_{z\in\mathcal{C}}\vert \Delta^{Q}_s(z)\vert
+
\frac{\rho^2}{\sqrt\lambda}\int_0^\tau ds\, \sup_{z\in\mathcal{C}} \vert \Delta^P_s(z)\vert \sup_{z\in\mathcal{C}}\vert Q_s(z)\vert
\nonumber \\ &
+
\frac{L}{\sqrt\lambda}\int_0^\tau ds\, \sup_{z\in\mathcal{C}} \vert \Delta^Q_s(z)\vert
\end{align}
and 
\begin{align}
\sup_{z\in\mathcal{C}}  &  \vert \Delta^{P}_\tau(z) \vert
\leq  
\sup_{z\in\mathcal{C}}\vert \Delta^{P}_0(z) \vert
+
\rho \int_0^\tau ds\, \sup_{z\in\mathcal{C}}\vert \Delta^{Q}_s(z)\vert^2
+
\rho\int_0^\tau ds\, \sup_{z\in\mathcal{C}}\vert \Delta^{Q}_s(z)\vert 
\sup_{z\in\mathcal{C}}\vert \bar Q_s(z) \vert
\nonumber \\ &
+
L \int_0^\tau ds\, \sup_{z\in\mathcal{C}}\vert \Delta^Q_s(z)\vert 
+ 
\frac{\rho}{\sqrt{\lambda}} \int_0^\tau ds\, \sup_{z\in\mathcal{C}}\vert \Delta^{P}_s(z)\vert
+
(1 + L) \rho \int_0^\tau ds\, \sup_{z\in\mathcal{C}}\vert\Delta^{Q}_s(z)\vert 
\sup_{z\in\mathcal{C}}\vert\Delta^{P}_s(z)\vert 
\nonumber \\ &
+
(1+L)\rho \int_0^\tau ds\, \sup_{z\in\mathcal{C}}\vert\Delta^{Q}_s(z)\vert \sup_{z\in\mathcal{C}}\vert \bar P_s(z) \vert
+
L^2 \int_0^\tau ds\, \sup_{z\in\mathcal{C}}\vert \Delta^P_s(z) \vert
\nonumber \\ &
+
\frac{\rho^2}{\sqrt\lambda} \int_0^\tau ds\, \sup_{z\in\mathcal{C}}\vert \Delta^{P}_s(z) \vert^2
+
\frac{\rho^2}{\sqrt\lambda} \int_0^\tau ds\, \sup_{z\in\mathcal{C}}\vert \Delta^{P}_s(z) \vert \sup_{z\in\mathcal{C}}\vert \bar P_s(z)\vert 
+
\frac{L}{\sqrt\lambda} \int_0^\tau ds\,\sup_{z\in\mathcal{C}}\vert \Delta^P_s(z) \vert
\end{align}
Now, using \eqref{eq:bound1} and \eqref{eq:bound2} we can "linearize" the right hand side to obtain
two inequalities of the form (where $C(\rho, \alpha, \lambda, T)$ is a suitable constant)
\begin{align}\label{eq:ineq1}
\sup_{z\in\mathcal{C}} \vert \Delta^Q_\tau(z)\vert  \leq \sup_{z\in\mathcal{C}} \vert \Delta^Q_0(z)\vert 
+ L\tau \sup_{z\in\mathcal{C}} \Delta^R(z)
+ C(\rho, \alpha, \lambda, T) \int_0^\tau ds\, \bigl\{ \sup_{z\in\mathcal{C}} \vert \Delta^Q_s(z)\vert + \sup_{z\in\mathcal{C}} \vert \Delta^P_s(z)\vert \bigr\}
\end{align}
and 
\begin{align}\label{eq:ineq2}
\sup_{z\in\mathcal{C}} \vert \Delta^P_\tau(z)\vert  \leq \sup_{z\in\mathcal{C}} \vert \Delta^P_0(z)\vert 
+ C(\rho, \alpha, \lambda, T) \int_0^\tau ds\, \bigl\{ \sup_{z\in\mathcal{C}} \vert \Delta^Q_s(z)\vert + \sup_{z\in\mathcal{C}} \vert \Delta^P_s(z)\vert \bigr\}
\end{align}
Summing \eqref{eq:ineq1} and \eqref{eq:ineq2} and iterating the resulting integral inequality we deduce
\begin{align}
\sup_{z\in\mathcal{C}} \vert \Delta^Q_\tau(z)\vert + \sup_{z\in\mathcal{C}} \vert \Delta^P_\tau(z)\vert 
\leq 
\bigl\{ \sup_{z\in\mathcal{C}} \vert \Delta^Q_0(z)\vert  + \sup_{z\in\mathcal{C}} \vert \Delta^P_0(z)\vert + LT\sup_{z\in\mathcal{C}} \Delta^R(z)\bigr\} e^{2 T C(\rho, \alpha, \lambda, T)}
\end{align}
By corollary \ref{stronger} we conclude that for $\tau\in [0, T]$
$\sup_{z\in\mathcal{C}} \vert \Delta^Q_\tau(z)\vert$ and $\sup_{z\in\mathcal{C}} \vert \Delta^P_\tau(z)\vert$ converge in $\mathbb{P}_\delta$-probability to zero. 

Finally, we can look at the overlaps. Observe that $\vert q(\tau) - \bar q(\tau)\vert  = \vert\int_{\mathcal{C}} \frac{dz}{2\pi i} \Delta^Q_\tau(z)\vert
\leq \rho \sup_{z\in\mathcal{C}} \vert \Delta^Q_\tau(z)\vert$ and 
$\vert p_1(\tau) - \bar p_1(\tau)\vert  = \vert\int_{\mathcal{C}} \frac{dz}{2\pi i} z \Delta^P_\tau(z)\vert
\leq \rho^2 \sup_{z\in\mathcal{C}} \vert \Delta^P_\tau(z)\vert$. Therefore $\vert q(\tau) - \bar q(\tau)\vert$ and $\vert p_1(\tau) - \bar p_1(\tau)\vert$ converge with $\mathbb{P}_\delta$-probability to $0$. But since $\lim_{n\to +\infty}\mathbb{P}(H\in \mathcal{S}_\delta^n) =1$ it is easy to see (by the law of total probability) that $\vert q(\tau) - \bar q(\tau)\vert$ and $\vert p_1(\tau) - \bar p_1(\tau)\vert$ also converge with $\mathbb{P}$-probability to $0$.

\end{proof}

\section{Enforcing the spherical constraint in gradient dynamics}\label{app:constraint}

The second term in equation \eqref{eq:mainode} enforces the spherical constraint $\theta_t\in \mathbb{S}^{n-1}(\sqrt n)$ at all times. This is well known but we briefly recall how to derive it for completeness. Since the $n$ dimensional sphere is embedded in $\mathbb{R}^n$ the covariant gradient $D_\theta$ can be obtained by projecting the usual gradient $\nabla_\theta$ on a tangent plane. This projection is obtained by subtracting the component along a radius of the sphere, i.e., 
$\frac{\theta}{\sqrt n}\langle \frac{\theta}{\sqrt n}, \nabla_\theta\mathcal{H}(\theta)\rangle$. Therefore gradient descent reads
\begin{align}
\frac{d\theta_t}{dt} = \eta D_\theta \mathcal{H}(\theta_t) = \eta\bigl( \nabla_\theta\mathcal{H}(\theta_t) -  \frac{\theta_t}{n}\langle \theta_t, \nabla_\theta \mathcal{H}(\theta_t)\rangle\bigr).
\end{align}
It is easily checked that $\frac{d \Vert \theta_t\Vert_2^2}{dt} = 0$ and since $\theta_0\in \mathbb{S}^{n-1}(\sqrt n)$ we have 
$\theta_t\in \mathbb{S}^{n-1}(\sqrt n)$ for all times. Indeed
\begin{align}
\frac{d\Vert \theta_t\Vert_2^2}{dt} 
= 2 \langle \theta_t, \frac{d\theta_t}{dt}\rangle 
= 2\eta\bigl( \langle \theta_t, \nabla_\theta\mathcal{H}(\theta_t)\rangle -  \frac{\langle\theta_t, \theta_t\rangle}{n} \langle \theta_t, \nabla_\theta\mathcal{H}(\theta_t)\rangle  \bigr) = 0.
\end{align}

\section{Strict saddle property}\label{app:landscape}

We say that the strict saddle property is satisfied if the critical points of the cost are {\it strict} saddles or minima (a strict saddle has by definition at least one strictly negative eigenvalue of the Hessian).
It is known from \cite{pmlr-v49-lee16} that for a cost satisfying the strict saddle property, gradient descent with small enough discrete time steps converges to a minimum, almost surely with respect to the initial condition. In the present context (as shown below) the critical points are given by the eigenvectors of $A\equiv\frac{\sqrt\lambda}{n}Y = \frac{\sqrt{\lambda}}{n}\theta^* \theta^{*T} + \frac{1}{\sqrt n}\xi$ - call them $v_i\in \mathcal{S}^{n-1}(\sqrt n)$, $i=1,\ldots, n$ - and the Hessian at $v_i$ is proportional to $\alpha_i I - A$ where $\alpha_i$ is the corresponding eigenvalue. For a random $n\times n$ matrix and {\it fixed} $\lambda$ the spectrum is almost surely non-degenerate,\footnote{However for a fixed realization when $\lambda$ varies we can have eigenvalue crossings.} i..e., $\alpha_1 < \alpha_2 < \ldots < \alpha_n$, so the strict saddle property is almost surely satisfied. Moreover the top eigenvector $v_n$ has positive definite Hessian and is a minimum, while for the other ones are strict saddles with non-zero positive and negative eigenvalues. Now, for 
$\lambda > 1$ we know, that for $n$ large enough with high probability, 
 $\{\alpha_1 < \cdots < \alpha_{n-1}\} \subset [-2, 2]$, $\alpha_n \approx \sqrt\lambda + 1/\sqrt\lambda > 2$ and 
$n^{-1}\vert\langle \theta_*, v_n\rangle\rangle\vert \approx \sqrt{1- 1/\lambda}$ (where $a\approx b$ means $\vert a - b\vert =o_n(1)$)
\cite{Pch2004TheLE, FP2006}. This explains that for $\lambda >1$ gradient descent with a small enough discrete time steps will converge to $v_n$ and the overlap approach $\pm \sqrt{1 -1/\lambda}$.

%As outlined below, using spectral properties of $Y$ \cite{Pch2004TheLE, FP2006}, it is 
%not difficult to show that $\mathcal{R}_H: \mathcal{S}\mapsto \mathbb{R}_+$ satisfies the strict saddle property, uniformly in $n$, with high probability with respect to the law of $\frac{1}{n}\xi$. More precisely, for finite $n$, with high probabability, for $\lambda > 1$ the cost has a single minimum equal to the top eigenvector of $Y$ and $n-1$ strict saddles. As a conseqence under gradient descent $\theta_t$ (with sufficiently small step size) the flows toward the top eigenvector of $Y$, as $t\to +\infty$, and the overlap is close to $\pm\sqrt{1-\lambda}$ (note that this argument is rigorous only in the regime of $n$ large but finite, because for $n\to +\infty$ the saddles although more numerous do not remain strict).

The critical points on the sphere $\mathcal{S}^{n-1}(\sqrt n)$ satisfy $D_\theta \mathcal{H}(\theta) = 0$ where 
$D_\theta = (1 - \frac{1}{n}\theta \theta^T)\nabla_\theta$ is the covariant derivative. We have
\begin{align}
D_\theta \mathcal{H}(\theta) 
\propto 
\frac{1}{n}\langle\theta, A \theta\rangle \theta  - A \theta = 0
\end{align}
and has $n$ solutions $\theta=v_i$, $i=1, \ldots, n$. The Hessian matrix on the sphere is (up to a positive prefactor)
\begin{align}
D_\theta D_\theta^T \mathcal{H}(\theta) 
 \propto 
   (1 - \frac{1}{n}\theta \theta^T) \bigl(\frac{1}{n}\langle \theta, A \theta\rangle I  - A\bigr)
\end{align}
and for each critical point $\theta = v_i$ we find $D_\theta D_\theta^T \mathcal{R}(v_i) \propto \frac{1}{n^2}(\alpha_i I - A)$. This has $n-1$ eigenvectors 
$v_j$, $j\neq i$ (perpendicular to $v_i$ and tangent to the sphere) with eigenvalues $\alpha_i - \alpha_j$, $j\neq i$, and one eigenvector $v_i$ with $0$ eigenvalue. For fixed $\lambda$ there is no degeneracy $\alpha_1 < \alpha_2 < \ldots < \alpha_n$, almost surely and $v_n$ is a minimum while 
$v_j$, $j\neq n$ are strict saddles.

%We now use the structure of the spectrum of $\frac{\sqrt{\lambda}}{n}Y = \frac{\sqrt{\lambda}}{n}\theta^* \theta^{*T} + \frac{1}{\sqrt n}\xi$ for $\lambda > 1$. With high probability when $n$ is large (but finite)
%we have $n-1$ non-degenerate eigenvalues 
%$\{\mu_1< \ldots <\mu_{n-1}\}\subset [-2, 2]$
%one top eigenvalue $\mu_n = \sqrt\lambda + \frac{1}{\lambda} > 2$, 
%and $n^{-1}\vert\langle \theta^*, v_n\rangle\vert \to \sqrt{1-\lambda}$. 
%Therefore when $\lambda > 1$ the Hessian $D_\theta D_\theta^T \mathcal{R}_H(v_n)$ is positive definite and the top eigenvector is a minimum. At the same time the eigenvalues of $D_\theta D_\theta^T \mathcal{R}_H(v_i)$, $i=1,\ldots, n-1$ have no definite sign and $v_i$, $i=1,\ldots, n-1$ are saddles. These saddles are strict because with high probability the eigenvalues are non-degenerate.
%
%

\section{Analysis of the stationary equation}\label{app:limiting-behavior-stoc}

The stationary equations corresponding to \eqref{eq:integrodiff} are given by setting the time derivatives on the left hand side to zero.
\begin{align}\label{eq:pde}
\begin{cases}
\bar q^\infty \left( \bar R(z) + \frac{1}{\sqrt \lambda} \right)
 + \left(\frac{z}{\sqrt \lambda} - (\bar q^\infty)^2 - \frac{1}{\sqrt \lambda} \bar p_1^\infty \right) \bar Q_\infty(z) = 0
 \\
 \bar q^\infty \bar Q_\infty(z) + \frac{1}{\sqrt \lambda}
 + \left( \frac{z}{\sqrt \lambda} - (\bar q^\infty)^2 - \frac{1}{\sqrt \lambda} \bar p_1^\infty \right) \bar P_\infty(z) = 0
\end{cases}
\end{align}
where $q^\infty \equiv - \int_{\mathcal{C}} \frac{dz}{2\pi i}\, Q_\infty(z)$, $p_1^\infty \equiv - \int_{\mathcal{C}} \frac{dz}{2\pi i}\, z P_\infty(z)$, $\bar R(z) = G_{\rm sc}(z)$, and $\mathcal{C} = \{z\in\mathbb{C}\mid z=\rho e^{i\theta}, \theta\in [0, 2\pi]\}$, $\rho >2$.
Here we show how to derive all possible solutions of these equations. 
One expects that the set of solutions contain the limiting solution for $\tau\to +\infty$ and we check that this is indeed the case. But it should be noted that there exist other solutions that are not physical in the sense that they are not attainable from the limiting time evolution.

From \eqref{eq:pde} we get
\begin{align}\label{eq:coupled}
\begin{cases}
\bar Q_\infty(z) = \bar q^\infty \frac{ \sqrt\lambda G_{\rm sc}(z) + 1}{\sqrt \lambda (\bar q^\infty)^2 + \bar p_1^\infty - z}
\\
\bar P_\infty(z) = (\bar q^\infty)^2 \sqrt \lambda \frac{ \sqrt\lambda G_{\rm sc}(z) + 1}{\left( \sqrt \lambda (\bar q^\infty)^2 + \bar p_1^\infty - z \right)^2} + \frac{1}{\sqrt \lambda (\bar q^\infty)^2 + \bar p_1^\infty - z}
\end{cases}
\end{align}
Let us first assume that $\vert \sqrt \lambda (\bar q^\infty)^2 + \bar p_1^\infty\vert \leq 2$. We integrate the second equation over the contour $\mathcal{C}$. One can show that integral of the first term on the right hand side vanishes. Thus we find the condition by $\bar p_1^\infty = \sqrt\lambda (\bar q^\infty)^2 + \bar p_1^\infty$ which implies $\bar  q^\infty = 0$. This implies in turn that $Q_\infty(z) = 0$, 
$P_\infty(z) = (\bar p_1^\infty - z)^{-1}$ and $\vert p_1^\infty\vert \leq 2$.

Now assume that $\vert \sqrt \lambda (\bar q^\infty)^2 + \bar p_1^\infty\vert \geq 2$. Integrating the first equation of \eqref{eq:coupled} over $\mathcal{C}$ we find
    \begin{equation}
        \bar q^\infty = \sqrt \lambda \bar q^\infty 
        \int_{z \in \mathcal{C}} \frac{dz}{2\pi i}
        \frac{ G_{\text{sc}}(z)
        }{
            z - (\sqrt\lambda\bar (q^\infty)^2 + \bar p_1^\infty)
        }
        \label{eq:first_eq_solve_q0}
    \end{equation}
The solution $\bar q^\infty = 0$ is again a possibility $Q_\infty(z) = 0$, 
$P_\infty(z) = (\bar p_1^\infty - z)^{-1}$ and $\vert \bar p_1^\infty\vert > 2$. 

Now assume that $\bar q^\infty\neq 0$ (and still $\vert\sqrt \lambda (\bar q^\infty)^2 + \bar p_1^\infty\vert \geq 2$). Computing the contour integral we find the equation $1 = - \sqrt\lambda G_{\rm sc}(\sqrt \lambda (\bar q^\infty)^2 + \bar p_1^\infty)$ and this implies 
\begin{equation}\label{eq:global}
\sqrt \lambda (\bar q^\infty)^2 + \bar p_1^\infty = \frac{1}{\sqrt\lambda} + \sqrt\lambda \geq 2
\end{equation}
for all $\lambda$.
Integrating the second equation in \eqref{eq:coupled} over $\mathcal{C}$ we find
 \begin{align}
        - \frac{1}{\lambda} = (\bar q^\infty)^2 
            \int_{z \in \mathcal{C}} \frac{dz}{2\pi i}
        \frac{ G_{\text{sc}}(z)
        }{
            \left( z - (\sqrt \lambda (\bar q^\infty)^2 + \bar p_1^\infty) \right)^2
        } = - (\bar q^\infty)^2 \frac{d G(z)}{dz}\vert_{\sqrt \lambda (\bar q^\infty)^2 + \bar p_1^\infty}
    \end{align}
Then using the explicit expression of $G(z)$ we find that $(\bar q^\infty)^2 = 1 - \frac{1}{\lambda}$, if  $\lambda \geq 1$, and $(\bar q^\infty)^2 = \frac{1}{\lambda}(\frac{1}{\lambda} - 1)$, if $\lambda \leq 1$.
Furthermore if 
$\lambda \geq 1$ we have from \eqref{eq:coupled} and \eqref{eq:global}
\begin{align}\label{eq:atlast1}
\begin{cases}
Q_\infty(z) = (1 - \frac{1}{\lambda}) \frac{ \sqrt\lambda G_{\rm sc}(z) + 1}{\sqrt\lambda + \frac{1}{\sqrt\lambda} - z}
\\
P_\infty(z) = (1 - \frac{1}{\lambda})  \sqrt\lambda \frac{ \sqrt\lambda G_{\rm sc}(z) + 1}{\left( \sqrt\lambda + \frac{1}{\sqrt\lambda} - z \right)^2} 
+ \frac{1}{\sqrt\lambda +\frac{1}{\sqrt\lambda} - z}
\end{cases}
\end{align}
Note that multiplying the second equation in \eqref{eq:atlast1} by $z$ and integrating over $\mathcal{C}$ yields $\bar p_1^\infty = \frac{2}{\sqrt\lambda}$. This consistent with \eqref{eq:global}.
Finally $\lambda < 1$ we have from \eqref{eq:coupled} and \eqref{eq:global} 
\begin{align}\label{eq:atlast2}
\begin{cases}
Q_\infty(z) = \frac{1}{\lambda}(\frac{1}{\lambda} - 1)\frac{ \sqrt\lambda G_{\rm sc}(z) + 1}{\sqrt\lambda + \frac{1}{\sqrt\lambda} - z}
\\
P_\infty(z) = \frac{1}{\lambda}(\frac{1}{\lambda} - 1) \sqrt\lambda \frac{ \sqrt\lambda G_{\rm sc}(z) + 1}{\left( \sqrt\lambda + \frac{1}{\sqrt\lambda} - z \right)^2} 
+ \frac{1}{\sqrt\lambda +\frac{1}{\sqrt\lambda} - z}
\end{cases}
\end{align}
As a consistency check we can see \eqref{eq:atlast2} and \eqref{eq:global} both imply $\bar p_1^\infty = \frac{2}{\sqrt\lambda} -\frac{1}{\lambda\sqrt\lambda} +\sqrt\lambda$. 

We conclude by noting that the solutions that are attainable from the time evolution are $(q^\infty = 0, \vert p_1^\infty\vert \leq 2)$ and
$(q^\infty = \pm \sqrt{1-\frac{1}{\lambda}}, p_1^\infty = \frac{2}{\lambda})$. The first one is "attained" from an initial condition 
with $\alpha = \frac{1}{n}\langle\theta^*, \theta_0\rangle =0$. In this case gradient descent "does not start" and $q^\infty = q(0) =0$, $p_1^\infty = p_1(0) = \frac{1}{n}\langle \theta_0, H \theta_0\rangle$ and $p_1(0) \leq 2$ with high probability. The other two solutions correspond to the initial conditions $\alpha = \frac{1}{n}\langle\theta^*, \theta_0\rangle$ with $\alpha >0$ and $\alpha <0$.

\section{Intermediate identities}\label{app:technical}

We derive a number of identities requiring interchange of integrals.

{\it A) Derivation of \eqref{eq:claimaboutQ}.}
To prove \eqref{eq:claimaboutQ} we start with \eqref{eq:qtlaplace} in the form
\begin{equation}
\mathcal L \hat Q_p(z) = \alpha G_{\text{sc}}(z) \mathcal L (e^{\frac{zt}{\sqrt\lambda}})(p)
 + \mathcal L \hat q(p) \mathcal L (e^{\frac{zt}{\sqrt\lambda}})(p) \bigl(G_{\text{sc}}(z) + \frac{1}{\sqrt \lambda}\bigr)
\end{equation}
and invert it back to the time domain
\begin{equation}\label{eq:genexplicit}
\hat Q_\tau(z) = \alpha G_{\text{sc}}(z) e^{\frac{z\tau}{\sqrt\lambda}}
+ G_{\text{sc}}(z)\int_0^\tau ds \hat q(s)  e^{\frac{z(\tau-s)}{\sqrt\lambda}}  + \frac{1}{\sqrt \lambda} \int_0^\tau ds \hat q(s)  e^{\frac{z(\tau-s)}{\sqrt\lambda}} .
\end{equation}
So this generating function is entirely known. Now we multiply this equation by $e^{\frac{z(\tau -u)}{\sqrt\lambda}}$ and integrate along $\mathcal{C}$. It is easy to see that, by Fubini's theorem, for the last term on the right hand side, the contour integral and the $s$-integral can be exchanged. Therefore the contour integral of the last term on the right hand side vanishes because $e^{\frac{z(2\tau-s -u)}{\sqrt\lambda}}$ is holomorphic in the whole complex plane.
For the other two terms on the right hand side we use the semi-circle law representation of $G_{\rm sc}(z)$ to obtain (see below for details) to obtain
\begin{align}\label{eq:prev}
\oint_{\mathcal{C}} \frac{dz}{2\pi i} G_{\text{sc}}(z) e^{\frac{z(2\tau-u)}{\sqrt\lambda}} = -  M_\lambda(2\tau - u)
\end{align}
and 
\begin{align}\label{eq:sim}
\oint_{\mathcal{C}} \frac{dz}{2\pi i} G_{\text{sc}}(z)\int_0^\tau ds \hat q(s)  e^{\frac{z(2\tau-s -u)}{\sqrt\lambda}}
= - \int_0^\tau ds \hat q(s) M_\lambda(2\tau - s - u) .
\end{align}
Putting  together \eqref{eq:prev}, \eqref{eq:sim} and \eqref{eq:genexplicit} we obtain the claimed identity \eqref{eq:claimaboutQ}.

{\it B) Derivation of \eqref{eq:prev}.}
From the semi-circle law representation of $G_{\rm sc}$
\begin{align}
\oint_{\mathcal{C}} \frac{dz}{2\pi i} G_{\text{sc}}(z) e^{\frac{z(2\tau-u)}{\sqrt\lambda}} 
= 
\oint_{\mathcal{C}} \frac{dz}{2\pi i} \int_{-2}^2 ds \frac{\mu_{\rm sc}(s)}{s-z} e^{\frac{z(2\tau-u)}{\sqrt\lambda}} 
\end{align}
It is easy to see that Fubini's theorem can be applied to interchange the integrals. Indeed the contour integral over 
$\mathcal{C}$ can be parametrized so that we then have two integrals with bounded functions over bounded intervals. So
\begin{align}
\oint_{\mathcal{C}} \frac{dz}{2\pi i} G_{\text{sc}}(z) e^{\frac{z(2\tau-u)}{\sqrt\lambda}} 
=
\int_{-2}^2 ds \mu_{\rm sc}(s) \oint_{\mathcal{C}} \frac{dz}{2\pi i} \frac{e^{\frac{z(2\tau -u)}{\sqrt\lambda}} }{s-z}
=
- \int_{-2}^2 ds \mu_{\rm sc}(s)  e^{\frac{s(2\tau-u)}{\sqrt\lambda}}
= - M_\lambda(2\tau - u)
\end{align}

{\it C) Derivation of \eqref{eq:sim}.}
We proceed similarly. First,
\begin{align}
\oint_{\mathcal{C}} \frac{dz}{2\pi i} G_{\text{sc}}(z) \int_0^\tau ds \hat q(s)  e^{\frac{z(2\tau-s -u)}{\sqrt\lambda}}
=
\oint_{\mathcal{C}} \frac{dz}{2\pi i} \int_{-2}^2 dx \int_0^\tau ds \mu_{\rm sc}(x) \hat q(s)  
\frac{e^{\frac{z(2\tau-s -u)}{\sqrt\lambda}}}{x-z}
\end{align}
Again, it is clear that the contour integral can be parametrized so that we all integrals are over bounded intervals and all functions are bounded, so that Fubini's theorem applies. Thus 
\begin{align}
\oint_{\mathcal{C}} \frac{dz}{2\pi i} G_{\text{sc}}(z)\int_0^\tau ds \hat q(s)  e^{\frac{z(2\tau-s -u)}{\sqrt\lambda}}
&
= \int_0^\tau ds \hat q(s) \int_{-2}^2 dx \mu_{\rm sc}(x) \oint_{\mathcal{C}} \frac{dz}{2\pi i} \frac{e^{\frac{z(2\tau-s -u)}{\sqrt\lambda}}}{x-z}
\nonumber \\ &
= 
- \int_0^\tau ds \hat q(s) M_\lambda(2\tau - s - u) 
\end{align}

{\it D) Derivation of \eqref{eq:almost}}. 
Again, using Fubini and then Cauchy's theorem,
\begin{align}
\oint_{\mathcal{C}} \frac{dz}{2\pi i} G_{\text{sc}}(z) \int_0^{+\infty} d\tau e^{-p\tau} e^{\frac{z\tau}{\sqrt \lambda}}
& =
\int_0^{+\infty} d\tau e^{-p\tau} \oint_{\Gamma^\prime} \frac{dz}{2\pi i} G_{\text{sc}}(z) e^{\frac{z\tau}{\sqrt \lambda}}
\nonumber \\ &
= 
\int_0^{+\infty} d\tau e^{-p\tau} 
\oint_{\mathcal{C}} \frac{dz}{2\pi i} e^{\frac{z\tau}{\sqrt \lambda}} \int_{-2}^2 ds \frac{\mu_{\text{sc}}(s)}{s - x}
\nonumber \\ &
= 
\int_0^{+\infty} d\tau e^{-p\tau} 
\int_{-2}^2 ds \mu_{\text{sc}}(s)
\oint \frac{dz}{2\pi i} \frac{e^{\frac{z\tau}{\sqrt \lambda}}}{s - z} 
\nonumber \\ &
= - \int_0^{+\infty} d\tau e^{-p\tau} \int_{-2}^2 ds \mu_{\text{sc}}(s) e^{\frac{s\tau}{\sqrt \lambda}} 
\nonumber \\ &
= - \int_0^{+\infty} d\tau e^{-p\tau}M_\lambda(\tau)
\end{align}

\section{Asymptotic analysis of $\bar q$}\label{app:asymptotic}

\subsection{limit when $\lambda > 1$}
We deduce the limiting behavior for $\lambda > 1$. The next order correction is given in \ref{app:second_order}.
Rewriting the first term from theorem \ref{th:risktrack}, we have for $\tau \in \mathbb R^+$
\begin{equation}
    e^{-(1+\frac{1}{\lambda}) \tau} \hat q(\tau) = \alpha \left[
        1 - \frac{1}{\lambda} \int_0^\tau e^{-(1+\frac{1}{\lambda}) s} M_\lambda(s) \dd s
    \right] .
\end{equation}
We notice that in the limit $t \to \infty$, the right hand side of the integral is the laplace transform
\begin{equation}
    \int_0^\infty e^{-(1+\frac{1}{\lambda}) s} M_\lambda(s) \dd s = \mathcal L M_\lambda \left(1+\frac{1}{\lambda} \right)
\end{equation}
and we have seen the connection with resolvent in \eqref{eq:remark} 
\begin{equation}
    \mathcal L M_\lambda \left(1+\frac{1}{\lambda} \right) = - \sqrt \lambda G_{\text{sc}}\left( (1+\frac{1}{\lambda})\sqrt \lambda\right) .
\end{equation}
But $X^2 + (1+\frac{1}{\lambda}) \sqrt \lambda X + 1 = 0$ has two roots: $\{ - \sqrt \lambda; \frac{-1}{\sqrt{\lambda}} \}$.
To ensure $G_{\text{sc}}(z) \in \mathbb C_+$ when $z \in \mathbb C_+$, 
we have $-\sqrt{\lambda}$ for $\lambda < 1$ and $\frac{-1}{\sqrt{\lambda}}$ for $\lambda > 1$.
Thus we conclude
\begin{equation}
    \lim_{\tau \to \infty} e^{-(1+\frac{1}{\lambda}) \tau}  \hat q(\tau) = 
    \left\{\begin{matrix}
        0 & (\lambda < 1)\\
        \alpha (1 - \frac{1}{\lambda}) & (\lambda > 1)
    \end{matrix}\right.
\end{equation}
Therefore, in the regime $\lambda > 1$, we find the asymptotic behavior for $\tau\to\infty$
\begin{equation}
    \hat q(\tau) \sim \alpha e^{(1+\frac{1}{\lambda}) \tau}  (1-\frac{1}{\lambda}) .
    \label{eq:q0_equivalence}
\end{equation}

A careful analysis of the terms entering $\hat p(\tau)$ shows the main contribution stems from the last term, on the square $\mathcal C = [\sqrt{\tau}, \tau]^2$ (as the integral can be neglected on $[0, \tau]^2 \setminus \mathcal C$):
\begin{equation}
    \hat p(\tau) \simeq \int_{\sqrt \tau}^\tau \int_{\sqrt \tau}^\tau  \hat q(u) q(v) M_\lambda(2\tau-u-v) \dd u \dd v .
\end{equation}
Using the approximation of $\hat q(t)$ in \eqref{eq:q0_equivalence} for large $t \in \mathcal C$, we can further approximate
\begin{equation}
    \hat p(\tau) \simeq \alpha^2 \left(1 - \frac{1}{\lambda} \right)^2 
    \int_{\sqrt \tau}^\tau \int_{\sqrt \tau}^\tau  
    e^{(1+\frac{1}{\lambda}) (u+v)\tau}
    M_\lambda(2\tau-u-v) \dd u \dd v
\end{equation}
and a change of variables $u=\tau-x, v=\tau-y$ provides
\begin{equation}
    \hat p(\tau) \simeq \alpha^2 
    e^{2 (1+\frac{1}{\lambda}) \tau}
    \left(1 - \frac{1}{\lambda} \right)^2 
    \iint_{[0,\tau(1-\frac{1}{\sqrt{\tau}})]^2} 
    e^{-(1+\frac{1}{\lambda}) (x+y)}
    M_\lambda(x+y) \dd y \dd x .
    \label{eq:first_p_equivalent}
\end{equation}
Now, notice the integral converges towards a non-zero value $K_\lambda$ when $\tau \to \infty$
\begin{equation}
    K_\lambda = \iint_{[0,\infty]^2} 
    e^{-(1+\frac{1}{\lambda}) (x+y)}
    M_\lambda(x+y) \dd y \dd x .
\end{equation}
Using a further change of variable $s = x+y$ we find 
%further by considering the integral surface:
\begin{equation}
    K_\lambda = 
    \int_{x=0}^\infty \int_{s = x}^\infty 
    e^{-(1+\frac{1}{\lambda}) s}
    M_\lambda(s) \dd s \dd x =
    \int_{s = 0}^\infty \int_{x=0}^s 
    e^{-(1+\frac{1}{\lambda}) s}
    M_\lambda(s) \dd x \dd s .
\end{equation}
Hence again, we find a connection with a Laplace transform (with a derivative from the additional $s$ term inside the integral)
\begin{equation}
    K_\lambda = 
    \int_{s = 0}^\infty 
    e^{-(1+\frac{1}{\lambda}) s}
    s M_\lambda(s) \dd x 
    = - (\mathcal L M_\lambda)' \left(1+\frac{1}{\lambda}\right)
\end{equation}
As $\mathcal L M_\lambda(p) = -\sqrt{\lambda} G_{\text{sc}}(p \sqrt{\lambda})$, and considering that $G_{\text{sc}} '(z) = -\frac{G_{\text{sc}}(z)}{2G_{\text{sc}}(z) + z }$, and that $G_{\text{sc}}( (1+\frac{1}{\lambda})\sqrt{\lambda}) = -\frac{1}{\sqrt{\lambda}}$ in the case when $\lambda > 1$, we conclude
\begin{equation}
    K_\lambda = 
    \lambda \frac{
        \frac{1}{\sqrt{\lambda}}
    }{ -2 \frac{1}{\sqrt{\lambda}}+ (1+\frac{1}{\lambda})\sqrt{\lambda} }
    = \frac{1}{1 - \frac{1}{\lambda}} .
    \label{eq:K}
\end{equation}
Finally, with \eqref{eq:K} and \eqref{eq:first_p_equivalent} we find
\begin{equation}
    \label{asym_p}
    \hat p(\tau) \sim \alpha^2 \left(1 - \frac{1}{\lambda} \right) e^{2 (1+\frac{1}{\lambda}) \tau}
\end{equation}
and for $\alpha>0$, we can conclude 
$\lim_{\tau \to \infty} \bar q(\tau) = \sqrt{1 - \frac{1}{\lambda}}$.

\subsection{Asymptotic analysis of $\lambda <1$} \label{app:ref_intervals}
The case $\lambda <1$ is computationally more involved as $\hat q(\tau)$ converges to $0$, and hence we need to find the rate of convergence towards $0$ of this term and that of $\hat p(\tau)$ in order to deduce the one from $\bar q(\tau)$. Though it is not the main topic of the paper, we provide some calculus elements to achieve this. We start with a lemma to find a suitable expression for $\hat q(\tau)$. Most of the calculations has been checked with Mathematica (the link will be provided in the final version).
%\url{https://www.wolframcloud.com/obj/antoine.bodin/Published/Check-Proofs-Gradient-Descent-Dynamic.nb}

\begin{lemma}
    $\hat q(\tau)$ has the following equivalent form:
    \begin{equation}
        \hat q(\tau) 
        =  \alpha \left(1-\frac{1}{\lambda}\right) e^{(1+\frac{1}{\lambda}) \tau}  \mathbb I_{(1,+\infty)}(\lambda)
        + \frac{2 \alpha}{\pi \lambda}  e^{ \frac{2}{\sqrt{\lambda}} \tau } \int_0^\pi 
        e^{ \frac{2}{\sqrt{\lambda}} (\cos(\theta) - 1) \tau } 
        \frac{\sin(\theta)^2}{(1+\frac{1}{ \lambda})  - \frac{2}{\sqrt{\lambda}} \cos(\theta)} \dd \theta
        \label{lemma:new_q}
    \end{equation}
\end{lemma}
\begin{proof}
Starting with $\hat q(\tau)$ from \eqref{th:risktrack}, one can use a similar expression of $M_\lambda$
\begin{equation}
    \frac{e^{-(1+\frac{1}{ \lambda}) \tau}}{\alpha} \hat q(\tau) 
    = 1 - \frac{2}{\pi \lambda} \int_0^\pi \int_0^\tau 
    e^{ \left(\frac{2}{\sqrt{\lambda}} \cos(\theta) - (1+\frac{1}{ \lambda})  \right) s}  \sin(\theta)^2 \dd s \dd \theta
\end{equation}
The inward integral can further be integrated (notice the constant term in the exponent is non-zero)
\begin{equation}
    \frac{e^{-(1+\frac{1}{ \lambda}) \tau}}{\alpha} \hat q(t) 
    = 1 -  \frac{2}{\pi \lambda} \int_0^\pi \left(e^{ \left(\frac{2}{\sqrt{\lambda}} \cos(\theta) - (1+\frac{1}{ \lambda}) \right) \tau } - 1 \right)  \frac{\sin(\theta)^2}{\frac{2}{\sqrt{\lambda}} \cos(\theta) - (1+\frac{1}{ \lambda})} \dd \theta
    \label{eq:q_eq_simplify}
\end{equation}
Using proposition \ref{prop:bioche} with the constant $a = \frac{1+\frac{1}{{\lambda}}}{\frac{2}{\sqrt \lambda}}>1$, one can simplify
\begin{equation}
    a - \sqrt{a^2 - 1}
    = \sqrt \lambda \frac{1 + \frac{1}{\lambda} - |1 - \frac{1}{\lambda}|}{2}
    = \left\{ \begin{matrix}
        \frac{1}{\sqrt \lambda} & (\lambda > 1)\\
        \sqrt \lambda & (\lambda < 1)\\
    \end{matrix} \right.
\end{equation}
and we finally find
\begin{equation}
    \frac{2}{\pi \lambda} \int_0^\pi \frac{\sin(\theta)^2}{\frac{2}{\sqrt{\lambda}} \cos(\theta) - (1+\frac{1}{ \lambda})} \dd \theta = 
    \frac{1}{\pi \sqrt{\lambda}} \int_0^\pi \frac{\sin(\theta)^2}{\cos(\theta) - a} \dd \theta 
    =  \left\{ \begin{matrix}
        -\frac{1}{\lambda} & (\lambda > 1)\\
        -1& (\lambda < 1)\\
    \end{matrix} \right.
    \label{eq:sim2_q}
\end{equation}
using the solution \eqref{eq:sim2_q} in \eqref{eq:q_eq_simplify} concludes the proof.
\end{proof}

\begin{proposition}
    \label{prop:bioche}
    For any $a > 1$, we have:
    \begin{equation}
        \int_0^\pi \frac{\sin(\theta)^2}{ \cos(\theta) - a} \dd \theta 
        =  \pi (\sqrt{a^2 - 1} - a)
    \end{equation}
    
\end{proposition}
\begin{proof}
Bioche's rules suggest a change of variable $u=\tan(\frac{\theta}{2})$, we find on the left-hand side
\begin{equation}
    \int_0^\pi \frac{\sin(\theta)^2}{ \cos(\theta) - a} \dd \theta = 
    \int_0^\infty \frac{ 
        \left( \frac{2u}{u^2+1} \right)^2
    }{
        \frac{1-u^2}{1+u^2} - a
    } \frac{2 \dd u}{1+u^2}
    = 
    \int_0^\infty \frac{ 
        8 u^2
    }{
        [ (1-a) - (1+a) u^2 ] (1+u^2)^2
    } \dd u
\end{equation}
Using the constant $K = \frac{a-1}{a+1}$ (or equivalently $a = \frac{1+K}{1-K}$) we can rewrite
\begin{equation}
    \int_0^\pi \frac{\sin(\theta)^2}{ \cos(\theta) - a} \dd \theta 
    = -4 (1-K)
    \int_0^\infty \frac{ 
     u^2
    }{
        (K + u^2) (1+u^2)^2
    } \dd u
\end{equation}
and make a classical partial fraction decomposition of the inward term of the integral
\begin{eqnarray}
    \frac{ 
        u^2
       }{
           (K + u^2) (1+u^2)^2
       }  
       & = & \frac{1}{(1-K)^2} \left( \frac{u^2}{K+u^2} - \frac{u^2}{1+u^2} \right) - \frac{1}{1-K} \frac{u^2}{(1+u^2)^2} \\
       & = & \frac{1}{(1-K)^2} \left( \frac{1}{1+u^2} -\frac{K}{K+u^2} \right) 
            - \frac{1}{1-K} \left[ \frac{1}{1+u^2} - \frac{1}{(1+u^2)^2}  \right] \\
       & = & \frac{K}{(1-K)^2} \left( \frac{1}{1+u^2} -\frac{1}{K+u^2} \right) 
            + \frac{1}{1-K} \frac{1}{(1+u^2)^2} 
\end{eqnarray}
Then on the one hand, with change of variable $u = \tan(x)$ we have:
\begin{equation}
    \int_0^\infty \frac{\dd u}{(1+u^2)^2} = \int_0^{\frac{\pi}{2}} \frac{\dd x}{1+\tan^2(x)}
    = \int_0^{\frac{\pi}{2}} \cos^2(x) \dd x = \frac{\pi}{4}
\end{equation}
On the other hand, with change of variable $u = \sqrt{K} \tan(x)$ we have:
\begin{equation}
    \int_0^\infty \frac{\dd u}{K+u^2} =  \int_0^{\frac{\pi}{2}} \dd x = \frac{\pi}{2} \frac{1}{\sqrt K}
\end{equation}
Thus:
\begin{equation}
    -4 (1-K) \int_0^\infty \frac{ 
        u^2 \dd u
       }{
           (K + u^2) (1+u^2)^2
       } =
       -\pi \left[
            \frac{2 K}{1-K} \left( 1 - \frac{1}{\sqrt K} \right) 
            + 1
       \right]
\end{equation}
and:
\begin{equation}
    \frac{2K}{1-K} \left( 1 - \frac{1}{\sqrt K} \right) + 1 = a - \sqrt{a^2 - 1}
\end{equation}
\end{proof}

Going back to the case $\lambda < 1$, we can simplify the expression from equation \eqref{lemma:new_q}
\begin{equation}
    \hat q(\tau) 
    =  \frac{2 \alpha}{\pi \lambda}  e^{ \frac{2}{\sqrt{\lambda}} \tau } \int_0^\pi 
    e^{ \frac{2}{\sqrt{\lambda}} (\cos(\theta) - 1) \tau } 
    \frac{\sin(\theta)^2}{(1+\frac{1}{ \lambda})  - \frac{2}{\sqrt{\lambda}} \cos(\theta)} \dd \theta
    \label{eq:simplify2}
\end{equation}
%With change of variable $x = 1 - \cos(\theta) $ we find also
%\begin{equation}
%    \hat q(\tau) 
%    =  \frac{2 \alpha}{\pi \lambda}  e^{ \frac{2}{\sqrt{\lambda}} \tau } \int_0^2 
%    e^{ -\frac{2}{\sqrt{\lambda}} x \tau } 
%    \frac{ 
%        (x(2-x))^\frac12
%    }
%    {
%        (1+\frac{1}{ \lambda})  - \frac{2}{\sqrt{\lambda}} (1-x)} \dd x
%\end{equation}
%Further, with $u=\frac{2}{\sqrt{\lambda}} x$ we have
Further, with $u=\frac{2}{\sqrt{\lambda}} (1-\cos(\theta))$ we rewrite \eqref{eq:simplify2} to apply Watson's lemma
\begin{equation}
    \hat q(\tau) 
    =  \frac{2 \alpha}{\pi \lambda}  e^{ \frac{2}{\sqrt{\lambda}} \tau } 
    \left(\frac{\sqrt{\lambda}}{2}\right)^\frac12
    \int_0^{\sqrt \lambda}
    e^{ -u \tau } 
    \frac{ 
        (u(2- \frac{\sqrt{\lambda}}{2}u))^\frac12
    }
    {
        (1+\frac{1}{ \lambda})  - \frac{2}{\sqrt{\lambda}} (1-\frac{\sqrt{\lambda}}{2}u)} \dd u
\end{equation}
Therefore, Watson's lemma provides the asymptotic equivalence
\begin{equation}
    \hat q(\tau) 
    \sim  
    \frac{2 \alpha}{\pi \lambda}  e^{ \frac{2}{\sqrt{\lambda}} \tau } 
    \left(\frac{\sqrt{\lambda}}{2}\right)^\frac32
    \frac{2^\frac12 \Gamma(\frac32) \tau^{-\frac32} }{(1+\frac{1}{ \lambda})  - \frac{2}{\sqrt{\lambda}}} 
\end{equation}
With $\Gamma(\frac32) = \frac{\sqrt{\pi}}{2}$ we have therefore
\begin{equation}
    \hat q(\tau) 
    \sim  
    \frac{\alpha \tau^{-\frac32} 
    e^{ \frac{2}{\sqrt{\lambda}} \tau } 
    }{2 \sqrt{\pi} \lambda^\frac14 \left(1 - \frac{1}{\sqrt{\lambda}}\right)^2}  
    \label{eq:watsonlemma}
\end{equation}

The remaining term $\hat p(\tau)$ can further be analyzed by splitting each integral from theorem \ref{th:risktrack} and analyzing the terms with the asymptotic form $e^{\frac{4\tau}{\sqrt\lambda }} \tau^{-\frac32}$.
For instance, we get easily the first term for which we have
\begin{equation}
    M_\lambda(2 \tau) = \frac{\sqrt{\lambda}}{2\tau} I_1\left( \frac{4\tau}{\sqrt{\lambda}}\right)
    \sim
    \frac{
        \sqrt{\lambda} 
        e^{\frac{4\tau}{\sqrt\lambda }}
    }{
        2t \sqrt{2 \pi \frac{4 \tau}{\sqrt \lambda } }
    }
    \sim 
    \frac{
        \lambda^\frac34 
        e^{\frac{4\tau}{\sqrt\lambda }}
    }{
        2^\frac52 \sqrt{\pi} \tau^\frac32
    }
    \label{eq:M_approx}
\end{equation}
The other terms require more technical considerations. We will use both former approximations from the equivalence relations \eqref{eq:watsonlemma} and \eqref{eq:M_approx}. However, these approximations are only valid for large $\tau$ while the integral for the second term is applied on the whole range $[0, \tau]$. Therefore, we split the integration intervals into two segments, say $[0, \sqrt \tau]$ and $[\sqrt \tau, \tau]$, and apply the approximations in the domains where it is valid. 

Starting with the second term, as $2\tau - s > \tau$ for all $s \in [0, \tau]$, we can already apply the relation \eqref{eq:M_approx} and split further the integrals:
\begin{equation}
    \int_0^\tau \hat q(s) M_\lambda(2\tau-s) \dd s \simeq
    \frac{\lambda^\frac34 e^{\frac{4}{\sqrt{\lambda}} \tau}}{2 \sqrt{\pi}}
    \left[
    \int_0^{\sqrt \tau} \hat q(s) 
    \frac{
        e^{-\frac{2}{\sqrt{\lambda}} s}
    }{(2\tau-s)^\frac32}
    \dd s 
    + 
    \int_{\sqrt \tau}^\tau \hat q(s) 
    \frac{
        e^{-\frac{2}{\sqrt{\lambda}} s}
    }{(2\tau-s)^\frac32}
    \dd s 
    \right]
    \label{eq:p_split}
\end{equation}
Then the integrand on the first segment of \eqref{eq:p_split} is further approximated using $\frac{1}{(2\tau-s)^\frac32} = \frac{1}{(2\tau)^\frac32}$. Indeed, as $s \leq \sqrt{\tau}$ we have $s = o(\tau)$. In the end we retrieve the laplace transform of $\hat q$:
\begin{equation}
    \int_0^{\sqrt \tau} \hat q(s) 
    \frac{
        e^{-\frac{2}{\sqrt{\lambda}} s}
    }{(2\tau-s)^\frac32}
    \dd s  \simeq 
    \frac{1}{(2\tau)^\frac32} 
    \int_0^{\sqrt \tau} \hat q(s) 
        e^{-\frac{2}{\sqrt{\lambda}} s}
    \dd s 
    \simeq
    \frac{1}{(2\tau)^\frac32} \mathcal L \hat q \left(\frac{2}{\sqrt \lambda }\right)
\end{equation}
From \eqref{eq:Lq0} which remains valid at $z=2$ with $G_{\text{sc}}(2) = -1$, we can even derive further the constant term
\begin{equation}
    \mathcal L \hat q \left(\frac{2}{\sqrt \lambda }\right) = 
    \frac{-\alpha G_{\text{sc}}(2)}{
        \frac{1}{\sqrt{\lambda}} + G_{\text{sc}}(2)
    } = \frac{\alpha}{\frac{1}{\sqrt \lambda} - 1}
\end{equation}
In the second segment of the integral in \eqref{eq:p_split}, we use the approximation from \eqref{eq:watsonlemma} and use change of variable $r = \frac{s}{\tau}$
\begin{equation}
    \int_{\sqrt \tau}^\tau \hat q(s) 
    \frac{
        e^{-\frac{2}{\sqrt{\lambda}} s}
    }{(2\tau-s)^\frac32}
    \dd s 
    \simeq 
    \frac{\alpha}{2 \sqrt{\pi} \lambda^\frac14
    \left[(1+\frac{1}{ \lambda})  - \frac{2}{\sqrt{\lambda}} \right] } 
    \int_{\frac{1}{\sqrt \tau}}^1 \frac{1}{r^\frac32 (2-r)^\frac32} \frac{\tau}{\tau^{\frac32 + \frac32}} \dd r
    \label{eq:p_eq_1}
\end{equation}
The integral from the right side can be solved:
\begin{equation}
    \int_{\frac{1}{\sqrt \tau}}^1 \frac{\dd r}{r^\frac32 (2-r)^\frac32} = \frac{
        1 - \frac{1}{\sqrt \tau}
    }{
        \sqrt{ 
            \frac{1}{\sqrt \tau}
            \left(2 - \frac{1}{\sqrt \tau}\right)
        }
    } \sim \frac{\tau^\frac14}{\sqrt 2}
    \label{eq:p_eq_2}
\end{equation}
Putting things together with \eqref{eq:p_eq_2} in \eqref{eq:p_eq_1} we get
\begin{equation}
    \int_{\sqrt \tau}^\tau \hat q(s) 
    \frac{
        e^{-\frac{2}{\sqrt{\lambda}} s}
    }{(2\tau-s)^\frac32} \dd s
    \simeq 
    \frac{\alpha \tau^{-\frac74} }{2 \sqrt{2\pi} \lambda^\frac14
    \left[(1+\frac{1}{ \lambda})  - \frac{2}{\sqrt{\lambda}} \right] } 
\end{equation}
So, the main contribution comes from the first integral of equation \eqref{eq:p_split} with the coefficient $\tau^{-\frac32}$
\begin{equation}
    2 \alpha \int_0^{\tau} \hat q(s) 
    M_\lambda(2\tau - s)
    \dd s 
    \sim
    \frac{
        \alpha^2 \lambda^\frac34 e^{\frac{4}{\sqrt{\lambda}} \tau} \tau^{-\frac32}
    }{
        2^\frac32 \sqrt{\pi} \left( \frac{1}{\sqrt \lambda} - 1 \right)
    }
    \label{eq:second_contrib}
\end{equation}

The third term with the double-integral requires extending the previous calculation idea on each rectangle:
$I_1 = [0,\sqrt{\tau}]^2$, $I_2=[0,\sqrt{\tau}]\times[\sqrt{\tau},\tau] $, $I_2'=[\sqrt{\tau},\tau]\times[0,\sqrt{\tau}] $ and $I_3 = [\sqrt{\tau},\tau]^2$. As we will see, only the integral on $I_1$ brings a contribution of order $\tau^{-\frac32}$ and the others can be neglected.

\paragraph{Interval $I_1 = [0, \sqrt \tau]^2$}
On this interval, $2\tau - u - v \gg 1$ so we consider
\begin{equation}
    \iint_{I_1} \hat q(u) \hat q(v) M_\lambda(2\tau - u -v) \dd u \dd v
    \simeq 
    \frac{\lambda^\frac34 e^{\frac{4}{\sqrt{\lambda}} \tau}}{2 \sqrt{\pi}}
    \iint_{I_1} \hat q(u) \hat q(v) \frac{e^{-\frac{2}{\sqrt \lambda}(u+v)}}{(2 \tau - u - v)^\frac32} \dd u \dd v
\end{equation}
also, on $I_1$ we have $\frac{1}{(2 \tau - u - v)^\frac32} \simeq \frac{1}{(2\tau)^\frac32}$, thus we are left to consider:
\begin{equation}
    \iint_{I_1} \hat q(u) \hat q(v) e^{-\frac{2}{\sqrt \lambda}(u+v)} \dd u \dd v
    \simeq \left[\mathcal L \hat q \left( \frac{2}{\sqrt \lambda} \right) \right]^2
\end{equation}
hence with \eqref{eq:q_eq_simplify} we find
\begin{equation}
    \iint_{I_1} \hat q(u) \hat q(v) M_\lambda(2\tau - u -v) \dd u \dd v
    \simeq 
    \frac{
        \alpha^2 \lambda^\frac34 e^{\frac{4}{\sqrt{\lambda}} \tau} \tau^{-\frac32}
    }{
        2^\frac52 \sqrt{\pi} \left( \frac{1}{\sqrt \lambda} - 1 \right)^2
    }
    \label{eq:third_contrib}
\end{equation}

\paragraph{Interval $I_2 = [0, \sqrt \tau] \times [\sqrt \tau, \tau]$} here we still have $2\tau - u - v \gg 1$ but also $v \geq \sqrt \tau \gg 1$ so with \eqref{eq:watsonlemma} we first get
\begin{equation}
    \iint_{I_2} \hat q(u) \hat q(v) M_\lambda(2\tau - u -v) \dd u \dd v
    \simeq 
    \frac{\alpha 
    }{2 \sqrt{\pi} \lambda^\frac14 \left[ (1+\frac{1}{ \lambda})  - \frac{2}{\sqrt{\lambda}}\right]}  
    \iint_{I_2} \hat q(u) v^{-\frac32} 
    e^{ \frac{2}{\sqrt{\lambda}} v } 
    M_\lambda(2\tau - u - v)
    \dd u \dd v
\end{equation}
and then:
\begin{equation}
    \iint_{I_2} \hat q(u) v^{-\frac32} 
    e^{ \frac{2}{\sqrt{\lambda}} v } 
    M_\lambda(2\tau - u - v)
    \dd u \dd v
    \simeq
    \frac{\lambda^\frac34 e^{\frac{4}{\sqrt{\lambda}} \tau}}{2 \sqrt{\pi}}
    \iint_{I_2} \hat q(u) v^{-\frac32} 
    e^{ \frac{2}{\sqrt{\lambda}} v } 
    \frac{e^{-\frac{2}{\sqrt \lambda}(u+v)}}{(2 \tau - u - v)^\frac32} \dd u \dd v
\end{equation}
so
\begin{equation}
    \iint_{I_2} \hat q(u) v^{-\frac32} 
    e^{ \frac{2}{\sqrt{\lambda}} v } 
    M_\lambda(2\tau - u - v)
    \dd u \dd v
    \simeq
    \frac{\lambda^\frac34 e^{\frac{4}{\sqrt{\lambda}} \tau}}{2 \sqrt{\pi}}
    \int_0^{\sqrt \tau} \hat q(u)e^{-\frac{2}{\sqrt \lambda}u}
    \int_{\sqrt \tau}^{\tau} 
    \frac{1}{v^\frac32 (2 \tau - u - v)^\frac32} \dd v
    \dd u
\end{equation}
At fixed $u \in [0, \sqrt \tau]$ With change of variable $s = \frac{v}{\tau}$ we find
\begin{equation}
    \int_{\sqrt \tau}^{\tau} 
    \frac{1}{v^\frac32 (2 \tau - u - v)^\frac32} \dd v
    = \int_{\frac{1}{\sqrt \tau}}^{1} 
    \frac{1}{\tau^\frac32 s^\frac32 (\tau (2-s) - u)^\frac32} \tau \dd s
    = \frac{1}{\tau^2} 
    \int_{\frac{1}{\sqrt \tau}}^{1} 
    \frac{1}{s^\frac32 \left((2-s) - \frac{u}{\tau}\right)^\frac32} \dd s
\end{equation}
Because $u \leq \sqrt \tau$ we have $\frac{u}{\tau} = o(1)$. Notice we have
\begin{eqnarray}
    \int_{\frac{1}{\sqrt \tau}}^{1} 
    \frac{\dd s}{s^\frac32 \left((2-s) - \frac{u}{\tau}\right)^\frac32} 
    & = & \left[ 
        \frac{
            2(\frac{u}{\tau} + 2s - 2)
        }{
            (\frac{u}{\tau} - 2)^2 \sqrt{ s (2 - s - \frac{u}{\tau}) }
        }
     \right]_{\frac{1}{\sqrt \tau}}^1 \\
     & = & \frac{2}{(\frac{u}{\tau} - 2)^2 }
     \left[ 
         \frac{u}{\tau \sqrt{1 - \frac{u}{\tau}}}
         - \frac{ \frac{u}{\tau} + \frac{2}{\sqrt \tau} - 2 }{
             \sqrt{ \frac{1}{\sqrt \tau} (2 - \frac{1}{\sqrt \tau} - \frac{u}{\tau})
             }
             }
     \right]
     \sim \frac{\tau^\frac14}{\sqrt 2}
\end{eqnarray}
Hence we have a term in $\tau^{-\frac74}$ so the term on $I_2$ can be neglected compared to $I_1$:
\begin{equation}
    \iint_{I_2} \hat q(u) v^{-\frac32} 
    e^{ \frac{2}{\sqrt{\lambda}} v } 
    M_\lambda(2\tau - u - v)
    \dd u \dd v
    \sim
    \frac{\lambda^\frac34 e^{\frac{4}{\sqrt{\lambda}} \tau} \tau^{-\frac74}}{2 \sqrt{2 \pi}}
    \mathcal L \hat q \left( \frac{2}{\sqrt \lambda} \right) 
    \label{eq:semgent_I2}
\end{equation}
Notice finally that the interval $I_2' = [\sqrt \tau, \tau] \times [0, \sqrt \tau] $ is similar as the integrand is symmetric in its arguments.

\paragraph{Interval $I_3 = [\sqrt \tau, \tau]^2$} we can approximate both $\hat q(u), \hat q(v)$
\begin{equation}
    \iint_{I_3} \hat q(u) \hat q(v) M_\lambda(2\tau - u -v) \dd u \dd v
    \simeq 
    \frac{\alpha^2
    }{4 \pi \lambda^\frac12 \left[ (1+\frac{1}{ \lambda})  - \frac{2}{\sqrt{\lambda}}\right]^2}  
    \iint_{I_3}  (uv)^{-\frac32} 
    e^{ \frac{2}{\sqrt{\lambda}} (u+v) } 
    M_\lambda(2\tau - u - v)
    \dd u \dd v
\end{equation}
Let's focus on the right hand side integral:
\begin{equation}
    f(\tau) = \iint_{I_3}  (uv)^{-\frac32} 
    e^{ \frac{2}{\sqrt{\lambda}} (u+v) } 
    M_\lambda(2\tau - u - v)
    \dd u \dd v
\end{equation}
Now, using change of variable $u = \tau - x, v = \tau - y$ we have:
\begin{equation}
    f(\tau)
    = e^{\frac{4}{\sqrt \lambda} \tau}
    \iint_{[0, \tau(1-\frac{1}{\sqrt \tau})]^2}
    (\tau - x)^{-\frac32}  (\tau - y)^{-\frac32}  e^{ \frac{-2}{\sqrt{\lambda}} (x+y) }  M_\lambda(x+y) \dd x \dd y
\end{equation}
with $s = x+y$:
\begin{eqnarray}
    e^{\frac{-4}{\sqrt \lambda} \tau}f(\tau)
    & = & 
    \int_0^{\tau(1-\frac{1}{\sqrt \tau})} \int_x^{x+\tau(1-\frac{1}{\sqrt \tau})}
    (\tau - x)^{-\frac32}  (\tau - s + x)^{-\frac32}  e^{ \frac{-2}{\sqrt{\lambda}} s }  M_\lambda(s) \dd s \dd x \\
    & = & 
    \int_0^{2\tau(1-\frac{1}{\sqrt \tau})} 
    \int_{\max\left(s - \tau(1-\frac{1}{\sqrt \tau}), 0\right)}^{\min\left(\tau(1-\frac{1}{\sqrt \tau}), s\right)}
    (\tau - x)^{-\frac32}  (\tau - s + x)^{-\frac32} \dd x e^{ \frac{-2}{\sqrt{\lambda}} s }  M_\lambda(s) \dd s \\
    & = & 
    \int_0^{\tau(1-\frac{1}{\sqrt \tau})} 
    \int_0^s
    (\tau - x)^{-\frac32}  (\tau - s + x)^{-\frac32} \dd x e^{ \frac{-2}{\sqrt{\lambda}} s }  M_\lambda(s) \dd s \\
    & + &
    \int_{\tau(1-\frac{1}{\sqrt \tau})}^{2\tau(1-\frac{1}{\sqrt \tau})} 
    \int_{s-\tau(1-\frac{1}{\sqrt \tau})}^{\tau(1-\frac{1}{\sqrt \tau})} 
    (\tau - x)^{-\frac32}  (\tau - s + x)^{-\frac32} \dd x e^{ \frac{-2}{\sqrt{\lambda}} s }  M_\lambda(s) \dd s
\end{eqnarray}
On the first integral, we find
\begin{equation}
    \int_0^s
    (\tau - x)^{-\frac32}  (\tau - s + x)^{-\frac32} \dd x 
    = \left[ \frac{2 (2x - s)}{
        (2\tau - s)^2 \sqrt{(\tau - x)(\tau + x - s)}
    } \right]_0^s
    = \frac{
        4s
    }{
        (2 \tau - s)^2 \sqrt{(\tau-s)\tau}
    }
\end{equation}
However, $s \leq \tau - \sqrt{\tau}$ so $\sqrt \tau \leq \tau - s$ and $\tau + \sqrt{\tau} \leq 2 \tau - s$ so:
\begin{equation}
    \frac{
        4s
    }{
        (2 \tau - s)^2 \sqrt{(\tau-s)\tau}
    } \leq 
    \frac{
        4 \tau (1 - \frac{1}{\sqrt \tau})
    }{
        \tau^2 (1 + \frac{1}{\sqrt{\tau}})^2 \tau^\frac12 \tau^\frac14
    }
    = 4 \tau^{-\frac74} (1 + o_\tau(1))
\end{equation}
Therefore, we find:
\begin{equation}
    \int_0^{\tau(1-\frac{1}{\sqrt \tau})} 
    \int_0^s
    (\tau - x)^{-\frac32}  (\tau - s + x)^{-\frac32} \dd x e^{ \frac{-2}{\sqrt{\lambda}} s }  M_\lambda(s) \dd s
    \leq 4 \tau^{-\frac74} (1  + o_\tau(1)) \mathcal L M_\lambda \left( \frac{2}{\sqrt \lambda} \right) 
\end{equation}
Noticeably, $\mathcal L M_\lambda \left( \frac{2}{\sqrt \lambda} \right) = -\sqrt \lambda G_{\text{sc}}(2) = \sqrt{\lambda}$. In the asymptotic limit, this term can be neglected due to $\tau^{-\frac74}$ compared to $\tau^{-\frac32}$.\\
Similarly, we find
\begin{equation}
    \int_{s-\tau(1-\frac{1}{\sqrt \tau})}^{\tau(1-\frac{1}{\sqrt \tau})}
    (\tau - x)^{-\frac32}  (\tau - s + x)^{-\frac32} \dd x
    = \frac{
        4 (2 \tau - s - 2 \sqrt{\tau})
    }{
        (2 \tau - s)^2 \tau^\frac14 \sqrt{ 2 \tau - \sqrt \tau - s}
    }
\end{equation}
then, in $[\tau - \sqrt{\tau}, 2(\tau - \sqrt{\tau})]$ we approximate $M_\lambda$ with its asymptotic expression. So we are left to evaluate 
\begin{equation}
    K(\tau) = \int_{\tau(1-\frac{1}{\sqrt \tau})}^{2\tau(1-\frac{1}{\sqrt \tau})} 
    \frac{
        4 (2 \tau - s - 2 \sqrt{\tau})
    }{
        s^\frac32 (2 \tau - s)^2 \tau^\frac14 \sqrt{ 2 \tau - \sqrt \tau - s} 
    }
    \dd s
\end{equation}
Notice that $2 \tau - s - 2 \sqrt{\tau} \leq \tau (1 - \frac{1}{\sqrt \tau})$, 
and $2 \tau^\frac12 \leq 2 \tau - s $ and $\tau^\frac12 \leq 2 \tau - \sqrt \tau - s$, hence
\begin{equation}
    0 \leq K(\tau) \leq 
    \frac{4 \tau (1 - \frac{1}{\sqrt \tau})}{ (2 \tau^\frac12)^2 \tau^\frac14 \sqrt{\tau^\frac12} }
    \int_{\tau(1-\frac{1}{\sqrt \tau})}^{2\tau(1-\frac{1}{\sqrt \tau})} 
    \frac{\dd s}{s^\frac32 }
\end{equation}
So
\begin{equation}
    0 \leq K(\tau) \leq 
    \frac{(1 - \frac{1}{\sqrt \tau})}{ \tau^\frac12 }
    \left[ -\frac{2}{s^\frac12 } \right]_{\tau(1-\frac{1}{\sqrt \tau})}^{2\tau(1-\frac{1}{\sqrt \tau})} 
\end{equation}

with a change of variable $u = s - (\tau - \sqrt{\tau})$ we find
\begin{equation}
    K(\tau) = \frac{1}{\tau^\frac14}
    \int_0^{\tau(1-\frac{1}{\sqrt \tau})} 
    \frac{
        4 ( \tau (1 - \frac{1}{\sqrt{\tau}}) - u )
    }{
        ( \tau (1 - \frac{1}{\sqrt{\tau}}) + u )^\frac32 
        (\tau(1  + \frac{1}{\sqrt{\tau}}) - u)^2  
        \sqrt{ \tau - u} 
    }
    \dd u
\end{equation}
with another change of variable $u = \tau r$ we find:
\begin{equation}
    K(\tau) = 
    \frac{
        4 
    }{
        \tau^\frac94
    }
    \int_0^{1-\frac{1}{\sqrt \tau}} 
    \frac{
        ( 1 - \frac{1}{\sqrt{\tau}} - r )
    }{
        ( 1 - \frac{1}{\sqrt{\tau}} + r )^\frac32 
        (1  + \frac{1}{\sqrt{\tau}} - r)^2  
        \sqrt{ 1 - r} 
    }
    \dd r
\end{equation}
Though this integral can be completely solved, we are only interested in bounding it. In particular, we find:
\begin{equation}
    K(\tau) \leq 
    \frac{
        4 
    }{
        \tau^\frac94
    }
    \int_0^{1-\frac{1}{\sqrt \tau}} 
    \frac{
        ( 1 - r )
    }{
        ( 1 - \frac{1}{\sqrt{\tau}} )^\frac32 
        (1  - r)^2  
        \sqrt{ 1 - r} 
    }
    \dd r
    =     \frac{
        4 
    }{
        \tau^\frac94 ( 1 - \frac{1}{\sqrt{\tau}} )^\frac32 
    }
    \int_0^{1-\frac{1}{\sqrt \tau}} 
    \frac{
      \dd r
    }{
        (1  - r)^\frac32
    }
\end{equation}
So 
\begin{equation}
    K(\tau) \leq 
    \frac{
        4 
    }{
        \tau^\frac94 ( 1 - \frac{1}{\sqrt{\tau}} )^\frac32 
    }
    \left[ \frac{
      2
    }{
        (1  - r)^\frac12
    } \right]_0^{1 - \frac{1}{\sqrt{\tau}} }
    = 
    \frac{
        8 (1 - \frac{1}{\tau^\frac14})
    }{
        \tau^\frac84 ( 1 - \frac{1}{\sqrt{\tau}} )^\frac32 
    }
    = 8 \tau^{-\frac84} (1 + o(1))
\end{equation}
In the end, the integral on $I_3$ can also be neglected.

\paragraph{conclusion} summing up all the main contributions from \eqref{eq:M_approx}, \eqref{eq:second_contrib} and \eqref{eq:third_contrib} we find
\begin{eqnarray}
    \lim_{\tau \to \infty} \tau^{\frac32}\e^{- \frac{4 \tau}{\sqrt{\lambda}}}\hat p(\tau) & = & 
    \frac{
        \lambda^\frac34 
    }{
        2^\frac52 \sqrt{\pi}
    }
    +
    \frac{
        \alpha^2 \lambda^\frac34
    }{
        2^\frac32 \sqrt{\pi} \left( \frac{1}{\sqrt \lambda} - 1 \right)
    }
    +
    \frac{
        \alpha^2 \lambda^\frac34 
    }{
        2^\frac52 \sqrt{\pi} \left( \frac{1}{\sqrt \lambda} - 1 \right)^2
    } \\
    & =&
    \frac{
        \lambda^\frac34 
    }{
        2^\frac52 \sqrt{\pi}
    } \left[
        1 + \alpha^2 \left(
            \frac{2}{\frac{1}{\sqrt \lambda} - 1}
            + \frac{1}{\left( \frac{1}{\sqrt \lambda} - 1 \right)^2}
        \right)
    \right] 
\end{eqnarray}
and thus:
\begin{equation}
    \frac{1}{\sqrt{\hat p(\tau)}}
    \sim 
    \frac{
        2^\frac54 \pi^\frac14
    }{
        \lambda^\frac38 
    } \left[
        1 - \alpha^2 + \frac{\alpha^2}{\lambda (\frac{1}{\sqrt \lambda} - 1)^2 }
    \right]^{-\frac12}
    \tau^{\frac34} e^{-\frac{2}{\sqrt \lambda} \tau}
\end{equation}
Using back \eqref{eq:watsonlemma} we find
\begin{equation}
    \label{asym_neg}
    \bar q(\tau)
    \sim  
    \frac{ \alpha \left( \frac{2}{\pi} \right)^\frac14  }{
        \lambda^\frac58 
        \left(1 - \frac{1}{\sqrt{\lambda}}\right)^2
        \sqrt{
        1 - \alpha^2 + \frac{\alpha^2}{\lambda (\frac{1}{\sqrt \lambda} - 1)^2 }
    }} \tau^{- \frac34}
\end{equation}
Numerical evaluations from the functions of theorem \ref{th:risktrack} match correctly this expression for different values of $(\alpha,\lambda)$, see figure \ref{fig:asymptotics} (a) for instance.

\subsection{Asymptotic analysis of $\lambda > 1$} \label{app:second_order}

Using the previous analysis for $\hat q(\tau)$ (\eqref{lemma:new_q} and \eqref{eq:watsonlemma}), we have an additional term:
\begin{equation}
    \hat q(\tau) = \alpha \left( 1 - \frac{1}{\lambda} \right)  e^{(1+\frac{1}{\lambda})\tau}
    +     \frac{\alpha \tau^{-\frac32} 
    e^{ \frac{2}{\sqrt{\lambda}} \tau } 
    }{2 \sqrt{\pi} \lambda^\frac14 \left(1 - \frac{1}{\sqrt{\lambda}}\right)^2}  
    + o( \tau^{-\frac32} e^{ \frac{2}{\sqrt{\lambda}} \tau } )
\end{equation}

Now, for $\hat p(\tau)$, we have already seen the leading asymptotics in equation \eqref{asym_p}. For the next correction,
we postulate through computer analysis that there exists a non-null constant $C \in \mathbb R_+^*$ such that it takes the form:
\begin{equation}
    \hat p(\tau) =  \alpha^2 \left( 1 - \frac{1}{\lambda} \right) e^{2(1+\frac{1}{\lambda})\tau} \left[ 1 
    - 2 \tau^{-\frac32} e^{-2(1-\frac{1}{\sqrt \lambda})^2 \tau} (C + o(1)) \right]
\end{equation}
Hence the expression:
\begin{equation}
    \frac{1}{\sqrt{\hat p(\tau)}} = 
    \frac{ e^{-(1+\frac{1}{\lambda})\tau}  }{|\alpha| \sqrt{1-\frac{1}{\lambda}}}
    \left[ 1 
    + \tau^{-\frac32} e^{-2(1-\frac{1}{\sqrt \lambda})^2 \tau} (C + o(1)) \right]
\end{equation}

Putting things together, we find:
\begin{equation}
    \bar q(\tau) = \text{sign}(\alpha) \sqrt{1-\frac{1}{\lambda}}
    \left( 
        1 + \frac{  \tau^{-\frac32} 
        e^{ -(1-\frac{1}{\sqrt \lambda})^2 \tau }  (1+o(1))
        }{2 (1-\frac{1}{\lambda}) \sqrt{\pi} \lambda^\frac14 \left(1 - \frac{1}{\sqrt{\lambda}}\right)^2}  
    \right)
    \left( 1 
    + \tau^{-\frac32} e^{-2(1-\frac{1}{\sqrt \lambda})^2 \tau} (C + o(1)) \right)
    \label{eq:bar_q_geq1}
\end{equation}
Hence the exponential term in the expression of $\hat q$ dominates the one in the expression of $\hat p$. Therefore, expanding the asymptotic expansion provides the result:
\begin{equation}
    \label{asym_pos}
    \bar q(\tau) - \text{sign}(\alpha) \sqrt{1-\frac{1}{\lambda}}
    \sim\frac{  \text{sign}(\alpha)  
        }{2 \sqrt{\pi} \lambda^\frac14 \sqrt{1-\frac{1}{\lambda}}  \left(1 - \frac{1}{\sqrt{\lambda}}\right)^2}  
        \tau^{-\frac32}  e^{ -(1-\frac{1}{\sqrt \lambda})^2 \tau } 
\end{equation}
More specifically, equation \eqref{eq:bar_q_geq1} shows that the second order term of $\hat q$ dominates the one of $\frac{1}{\sqrt{\hat p}}$ when we compute the final contribution in equation \eqref{asym_pos}. Therefore, this fact can be emphasized with the equivalent limiting behavior:
\begin{equation}
    \label{asym_lim}
        \bar q(\tau) - \text{sign}(\alpha) \sqrt{1-\frac{1}{\lambda}}
   \sim
        \frac{1}{|\alpha| \sqrt{1 - \frac{1}{\lambda}}} \left(\hat q(\tau) e^{-(1+\frac{1}{\lambda})\tau} -  \alpha \left(1 - \frac{1}{\lambda}\right)\right)
\end{equation}
This form is actually more convenient because a numerical evaluation $\bar q(\tau)$ for large $\tau$ requires extra precision and computational resources due to the double-integral within the $\hat p(\tau)$ term. Therefore, it appears to be easier to observe the equivalent behavior in \eqref{asym_lim} rather than in \eqref{asym_pos}.
To illustrate this phenomenon, one can evaluate:
\begin{eqnarray}
    \psi(\tau) & = &  
    |\alpha| \sqrt{1 - \frac{1}{\lambda}}
    \left( \bar q(\tau) - \text{sign}(\alpha) \sqrt{1-\frac{1}{\lambda}} \right) e^{(1-\frac{1}{\sqrt \lambda})^2 \tau}\\
    \phi(\tau) & = & \left(\hat q(\tau) e^{-(1+\frac{1}{\lambda})\tau} -  \alpha \left(1 - \frac{1}{\lambda}\right)\right) e^{(1-\frac{1}{\sqrt \lambda})^2 \tau}\\
    \mathcal A(\tau) & = & \frac{  \alpha
    }{2 \sqrt{\pi} \lambda^\frac14   \left(1 - \frac{1}{\sqrt{\lambda}}\right)^2}  
    \tau^{-\frac32} 
\end{eqnarray}
and expect to observe $\psi(\tau) \sim \phi(\tau) \sim \mathcal A(\tau)$ when $\tau \to \infty$ for any $\lambda >1$ and $\alpha \neq 0$. See figure \ref{fig:asymptotics} (b) as an example where the computation of $\psi(\tau)$ had to be stopped earlier in time to cope with computational limits of the math library Scipy.

\begin{figure}
    \centering
    \subfigure[$\lambda=0.5$, $\alpha = 0.1$]{\includegraphics[width=7cm]{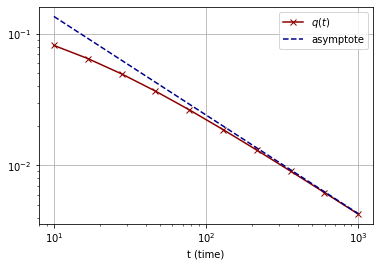} }
    \qquad
    \subfigure[$\lambda=5$, $\alpha = 0.1$]{\includegraphics[width=7cm]{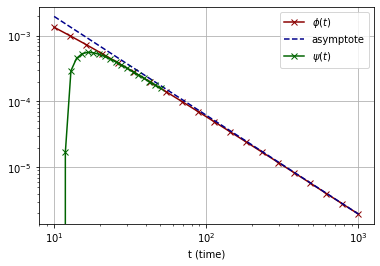} }
    \caption{
        Example of a numerical evaluation of theorem \ref{th:risktrack} and comparisons with their respective asymptotes in log-scale for $\lambda < 1$ in (a) and $\lambda > 1$ in (b).
    }
    \label{fig:asymptotics}
\end{figure}

\subsection{Asymptotic analysis for $\lambda = 1$}
In the special case $\lambda = 1$ where the regime changes, one can write explicitly:
\begin{equation}
    \hat q(\tau) = \alpha e^{2\tau} \left( 
    1 - 1 + e^{-2\tau}(I_0(2\tau) + I_1(2\tau))
    \right) = \alpha \left[
        I_0(2\tau) + I_1(2\tau)
    \right]
\end{equation}
and we find the first term of the asymptotic expansion in $\tau \to \infty$:
\begin{equation}\label{qchapeau}
    \hat q(\tau) \sim \alpha \frac{e^{2\tau}}{\sqrt{\pi \tau}} .
\end{equation}
Some further analysis lead us to a similar estimate for $\hat p(\tau)$
\begin{equation}
    \sqrt{\hat p(\tau)} \sim |\alpha| \frac{e^{2t}}{(2 \pi \tau)^{\frac14}}
\end{equation}
and thus to conclude using \eqref{qchapeau} (for $\alpha > 0$):
\begin{equation}
    \bar q(\tau) \sim \left( \frac{2}{\pi \tau} \right)^\frac14
\end{equation}

Using similar arguments as the case $\lambda < 1$ (see section \ref{app:ref_intervals}), we can check that the main asymptotic contribution in $\tau^{-\frac12}$ comes from the the third term of $\hat p$ on the interval $I_3$. 
Indeed, the first term in $M_1(2\tau)$ is obviously in $\tau^{-\frac32}$.%give ref
The second term can also be neglected, notice that we have:
\begin{equation}
    \int_{\sqrt \tau}^\tau \hat q(s) \frac{e^{-2s}}{(2\tau-s)^\frac32} \dd s \sim
    \frac{\alpha}{\sqrt{\pi}}
    \int_{\sqrt \tau}^\tau 
    \frac{\dd s}{\sqrt s (2\tau - s)^\frac32}
    \sim 
    \frac{\alpha}{\sqrt \pi \tau}
\end{equation}
Also we don't have a constant term with the laplace transform of $\hat q$. Instead  for any $t>0$
\begin{equation}
    \int_0^t \hat q(s) e^{-2s} \dd s
    = \frac{\alpha}{2} (e^{-2 t} (1 + 4 t) I_0(2 t) + 
   4 t e^{-2 t} I_1(2 t) - 1)
\end{equation}
In particular when $t=\sqrt \tau$ and $\tau \to \infty$:
\begin{equation}
    \int_0^{\sqrt \tau} \hat q(s) e^{-2s} \dd s
    \sim  \frac{4 \alpha \sqrt{\tau}}{\sqrt{4 \pi \sqrt{\tau}}} \sim
    \frac{2 \alpha \tau^\frac14}{\sqrt{\pi }}
    \label{eq:lq_1_asym}
\end{equation}
Hence with the additional term in $\tau^{-\frac32}$ this gives a term in $\tau^{-\frac54}$.
We proceed similarly for the third term with the $4$ segments $I_1, I_2, I_2', I_3$.
\paragraph{Interval $I_1 = [0, \sqrt \tau]^2$} Similar considerations using the result \eqref{eq:lq_1_asym} lead to the asymptotics:
\begin{equation}
    \iint_{I_1} \hat q(u) \hat q(v) e^{-2(u+v)} \dd u \dd v \sim
    \frac{4 \alpha^2 \tau^\frac12}{\pi }
\end{equation}
Hence with the additional term in $\tau^{-\frac32}$ this gives a term in $\tau^{-1}$.

\paragraph{Interval $I_2 = [0, \sqrt \tau] \times [\sqrt \tau, \tau]$}
We get:
\begin{equation}
    \iint_{I_2} \hat q(u) \hat q (v) M_1(2\tau - u -v) \dd u \dd v \simeq
    \frac{\alpha}{2 \pi} \iint_{I_2} \hat q(u) \frac{e^{2v}}{\sqrt v} 
    \frac{ e^{2(2\tau - u -v)}}{ (2\tau-u-v)^\frac32}\dd u \dd v
\end{equation}
We can compute further the integral considering $u = o(\tau)$:
\begin{equation}
    \begin{array}{ccl}
        \int_{v} \frac{1}{\sqrt{v} (2 \tau - u -v)^\frac32} \dd v
        & = &
        \frac{2}{2 \tau - u}
        \left[
            \sqrt \frac{v}{2\tau - u - v}
        \right]_{\sqrt \tau}^\tau \\
        & = & 
        \frac{2}{2 \tau - u} \left[
        \sqrt{ \frac{\tau}{2\tau - u} } - 
        \sqrt{ \frac{\sqrt{\tau}}{2\tau - u - \sqrt{\tau}} } \right] \\
        & \sim & \tau^{-1}
    \end{array}
\end{equation}
Finally, using \eqref{eq:lq_1_asym} gives:
\begin{equation}
    \iint_{I_2} \hat q(u) \hat q (v) M_1(2\tau - u -v) \dd u \dd v 
    \sim
    \frac{\alpha^2}{\pi^\frac32}  e^{4\tau} \tau^{-\frac34}
\end{equation}

\paragraph{Interval $I_3 = [\sqrt \tau, \tau]^2$}
On this interval we have:
\begin{equation}
    \iint_{I_3} \hat q(u) \hat q (v) M_1(2\tau - u -v) \dd u \dd v \simeq
    \frac{\alpha^2 }{\pi}
    \iint_{I_3} e^{2(u+v)} \frac{I_1(2(2\tau - u - v))}{(2\tau - u -v)\sqrt{uv}} \dd u \dd v
\end{equation}
Let's focus on the right hand side integral:
\begin{equation}
    f(\tau) = \iint_{I_3} e^{2(u+v)} \frac{I_1(2(2\tau - u - v))}{(2\tau - u -v)\sqrt{uv}} \dd u \dd v
\end{equation}
With $x = \tau - u$, $y = \tau - v$ we find:
\begin{equation}
    e^{-4 \tau} f(\tau)
    = 
    \iint_{[0,\tau - \sqrt \tau]^2} e^{-2(x+y)} \frac{I_1(2(x+y))}{(x+y)\sqrt{(\tau - x)(\tau - y)}} \dd x \dd y
\end{equation}
Now, consider further the change of variable: $x=(\tau - \sqrt{\tau} ) r$ and $y=(\tau - \sqrt{\tau} ) s$ .
we have:
\begin{equation}
    \sqrt{\tau} e^{-4 \tau} f(\tau)
    = 
    \iint_{[0,1]^2} 
    \sqrt{\tau}
    \frac{
        e^{-2(\tau - \sqrt{\tau} )(r+s)} I_1(2(\tau - \sqrt{\tau} )(r+s))
    }{
        (r+s)\sqrt{ \left( \frac{1}{1 - \frac{1}{\sqrt{\tau}}} - r \right)\left(\frac{1}{1 - \frac{1}{\sqrt{\tau}}} - s\right)}
    } \dd r \dd s
\end{equation}
Now, for all $r,s \in [0,1]^2\setminus\{(0,0)\}$, we have:
\begin{equation}
    \lim_{\tau \to \infty} \sqrt{\tau}
    \frac{
        e^{-2(\tau - \sqrt{\tau} )(r+s)} I_1(2(\tau - \sqrt{\tau} )(r+s))
    }{
        (r+s)\sqrt{ \left( \frac{1}{1 - \frac{1}{\sqrt{\tau}}} - r \right)\left(\frac{1}{1 - \frac{1}{\sqrt{\tau}}} - s\right)}
    } = \frac{1}{\sqrt{4\pi}(r+s)^\frac32 \sqrt{(1-r)(1-s)}}
\end{equation}
and it can be shown that this function is integrable:
\begin{equation}
    \iint_{[0,1]^2} 
    \frac{\dd r \dd s}{\sqrt{4\pi}(r+s)^\frac32 \sqrt{(1-r)(1-s)}}
      =  \sqrt \frac{\pi}{2}
\end{equation}
Further, for all $r,s \in [0,1]^2\setminus\{(0,0)\}$ and for instance $\tau \geq 4$:
\begin{equation}
    \sqrt{\tau} I_1(2(\tau - \sqrt{\tau} )(r+s))
    \leq \frac{1}{\sqrt{4 \pi (1 - \frac{1}{\sqrt{\tau}} )(r+s)}} \leq \frac{\sqrt 2}{\sqrt{4 \pi(r+s)}}
\end{equation}
and 
\begin{equation}
    \frac{1}{\sqrt{ \left( \frac{1}{1 - \frac{1}{\sqrt{\tau}}} - r \right)\left(\frac{1}{1 - \frac{1}{\sqrt{\tau}}} - s\right)}}  \leq 
    \frac{1}{\sqrt{(1-r)(1-s)}}
\end{equation}
Hence for all $\tau \geq 4$, the integrand is dominated by its limit times $\sqrt 2$.

In conclusion, we have the main contribution term
\begin{equation}
    \iint_{I_3} \hat q(u) \hat q (v) M_1(2\tau - u -v) \dd u \dd v \sim \frac{\alpha^2}{\sqrt{2 \pi \tau}}
\end{equation}

\subsection{Asymptotic analysis conclusion}
We have seen the case $\lambda < 1$ in \eqref{asym_neg} and $\lambda > 1$ in \eqref{asym_pos}. 
So compared to the first case $\lambda < 1$, the convergence towards the limit is reached with an exponential term $\exp\{ -(1-\frac{1}{\sqrt \lambda})^2 \tau \} $ in the asymptotic limit for $\lambda > 1$. It confirms the result that the convergence happens faster as $\lambda$ grows to infinity, and that the exponential term vanishes as $\lambda$ gets close to $1$ - with an additional singularity in the denominator.

% Further, with $s=x+y$ we have:
% \begin{eqnarray}
%     e^{-4 \tau} f(\tau) & = & 
%     \int_{0}^{\tau - \sqrt \tau} \int_{x}^{x+\tau-\sqrt \tau}
%     \frac{e^{-2s} I_1(2s)}{s \sqrt{(\tau -x )(\tau + x - s)}} \dd s \dd x\\
%     & = & \int_{0}^{\tau - \sqrt \tau} \int_0^{s} \frac{e^{-2s} I_1(2s)}{s \sqrt{(\tau -x )(\tau + x - s)}} \dd x \dd s \\
%     & + & \int_{\tau - \sqrt \tau}^{2(\tau - \sqrt \tau)} \int_{s-(\tau - \sqrt \tau)}^{\tau - \sqrt \tau} \frac{e^{-2s} I_1(2s)}{s \sqrt{(\tau -x )(\tau + x - s)}} \dd x \dd s \\
% \end{eqnarray}
% In both sub-integrals, we have:
% \begin{equation}
%     \int \frac{\dd x}{\sqrt{(\tau -x )(\tau + x - s)}}
%     = 2 \sinh^{-1}\left(
%         \sqrt \frac{\tau - x}{2\tau - s}
%     \right)
% \end{equation}

% We have:
% \begin{equation}
%     \iint_{I_3} \hat q(u) \hat q (v) M_1(2\tau - u -v) \dd u \dd v \simeq
%     \frac{\alpha^2 }{\pi} 
%     \iint_{I_3} \frac{e^{2(u+v)} }{\sqrt{uv}}
%     M_1(2\tau - u - v) \dd u \dd v
% \end{equation}
% With change of variable $x = \tau - u$ and $y = \tau - v$, and defining $t= \tau(1 - \frac{1}{\sqrt \tau})$ we have:
% \begin{equation}
%     \iint_{I_3} e^{2(u+v)} 
%     M_1(2\tau - u - v) \dd u \dd v
%     = e^{4\tau}
%     \iint_{[0,t]^2} \frac{e^{-2(x+y)} 
%     M_1(x+y)}{\sqrt{(\tau - x)(\tau - y)}} \dd x \dd y
% \end{equation}
% Now we have (with $r=2-s$)
% \begin{eqnarray}
%     \iint_{[0,t]^2} e^{-2(x+y)} 
%     M_1(x+y) \dd x \dd y
%     & = & \int_{-2}^2 \mu_{\text{sc}}(s) \iint_{[0,t]^2} \frac{e^{-(2-s)(x+y)}}{\sqrt{(\tau - x)(\tau - y)}} \dd x \dd y \dd s
% \end{eqnarray}
\section{Additional experiments}\label{app:add-experiments}
A link to reproduce all the examples will be available in the final version.
%All the examples are available at the following location:
%\url{https://github.com/antoinexp/matrix-gradient-descent-dynamic/blob/main/gradient-descent-dynamic-experiments.ipynb}

\subsection{Limiting gradient descent}
We illustrate the predicted time evolution for cases $\alpha$ very close to $0$ and $\alpha$ very close to $1$ in figure \ref{fig:add_theoretical_curves}. Since $\alpha = 0$ leads to a null overlap evolution, a slight non-zero initial value of $\alpha$ is required to initiate the learning algorithm. The smaller the $\alpha$ the more is the asymptotic regime delayed. The opposite case $\alpha = 1$ brings another insight, namely 
when $\theta_0 = \pm \theta^*$ the effect of the noise inexorably disturbs the signal towards a lower limiting overlap (for $\lambda < \infty$).
\begin{figure}
    \centering
    \subfigure[$\alpha = 0.01$]{\includegraphics[width=7cm]{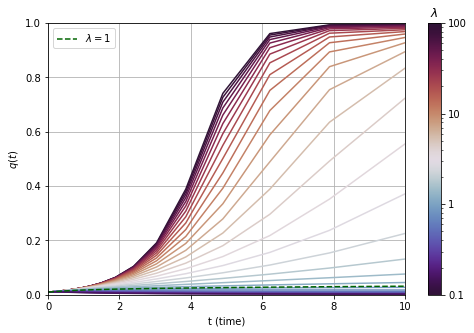} }
    \qquad
    \subfigure[$\alpha = 1.0$]{\includegraphics[width=7cm]{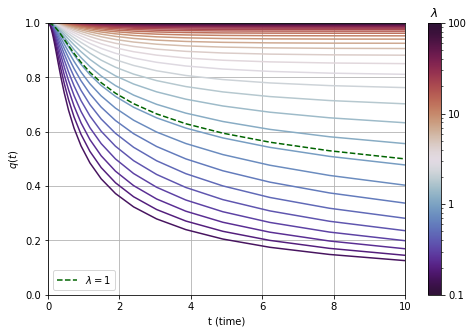} }
    \caption{Comparison of the overlap over time with different configurations of $\lambda$ parameter, and between two different values of $\alpha$.}
    \label{fig:add_theoretical_curves}
\end{figure}

\subsection{Comparison with experimental gradient descent algorithm}\label{subsubsec:experiments}

%by performing multiple runs on vectors $\theta_t \in \mathbb R^n$ on $t\in \{t_1, \ldots, t_k, \ldots\}$ for a sufficiently large dimension $n$ and sampling over random matrices $\xi$. The numerical codes to reproduce the graphs of this section is provided in the Appendix \ref{app:add-experiments} and can be run in a few minutes on a standard desktop setting.
%

The theoretical gradient descent prediction is compared with the experimental values when taking the data dimension $n$ sufficiently large over multiple runs with new samples of the noise matrix. Discrete step size gradient descent is performed while keeping $\theta_t$ on $\mathcal S_d(\sqrt n)$. We choose a $\delta_t > 0$ sufficiently small and consider discrete times $t_k = k \delta_t$ for $k \in \mathbb N$.  We update $\theta_{t_k}$  in two steps: first with the gradient descent
$\theta_{t_k+\frac{\delta t}{2}} = \theta_{t_k} - \eta \delta_t \nabla \mathcal H(\theta_t)$, and secondly projecting back on the sphere $\theta_{t_{k+1}} = \sqrt{n} \theta_{t_k+\frac{\delta t}{2}} \lVert \theta_{t_k+\frac{\delta t}{2}}\rVert^{-1}$. These steps are implemented using Tensorflow in Python and run seamlessly on a standard single computer configuration. The initial vectors $\theta_0$ and $\theta^*$ are chosen deterministically as $\sqrt n \theta_0 = \alpha \e_1 + \sqrt{1-\alpha^2} \e_2 $ and $\sqrt n \theta^* = \e_1$ with $(e_i)_{1 \leq i \leq n}$ the canonical basis of $\mathbb R^n$, while the noise matrix $H$ is generated randomly. To account for the randomness of $H$ at each execution, we perform $100$ runs and give the quantiles for quantities of interest.

\begin{figure}
    \centering
    \subfigure[Overlap]{\includegraphics[width=6cm]{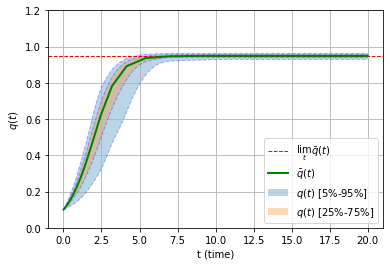} }
    \qquad
    \subfigure[Cost]{\includegraphics[width=6cm]{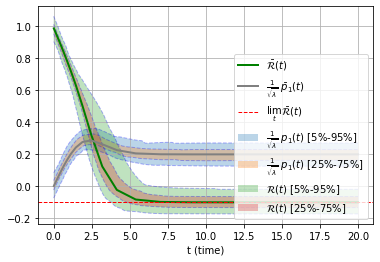} }
\caption{$\lambda = 10, n = 70, \alpha = 0.1, \delta_t=0.1$}
\label{fig:experimental_runs}
\end{figure}
As shown in figure  \ref{fig:experimental_runs}, the learning curve matches the theoretical limiting curve with some fluctuations. As illustrated below, these fluctuations diminish as $n$ is increased. Noticeably, in the regime where $\lambda > 1$, smaller values of $\lambda$ require higher values of $n$ to keep the same concentration.
Therefore, the formula from theorem \ref{th:risktrack} provides a good theoretical framework to predict the behavior of the experimental learning algorithm. Such formulas potentially allow to benchmark the time-evolution of gradient descent techniques and provide guidelines for early-stopping commonly used in machine learning.

We provide a range of further different experiments for different values of $\lambda$, $\alpha$, $n$. 

%These experiments are executed with multiple runs to preclude potential randomness effects of $H$.
Let us first comment the regime $\lambda >1$ illustrated on figures \ref{fig:experimental_runs2}, \ref{fig:experimental_runs3}, 
\ref{fig:experimental_runs5}.
Figure \ref{fig:experimental_runs2} clearly shows that increasing $n$ up to $1000$ concentrates the experimental curves  around the expect limiting overlap and cost $\bar q, \bar{ \mathcal{H}}$. We also see even more clearly the characteristic change of $p_1$ with a "self-healing" process at some specific point in the dynamics of the learning algorithm
(recall that $p_1$ is a similarity measure between the reconstructed matrix $\theta_t\theta_t^T$ and the noise matrix $H$). This is also seen in Figures \ref{fig:experimental_runs3} and \ref{fig:experimental_runs5} for different values of $\lambda$ and $\alpha$.
Figures \ref{fig:experimental_runs2} and \ref{fig:experimental_runs3} only differ in the value $\lambda$: we observe that decreasing this parameter closer to $1$ not only decreases the overlap, but also increases the deviation from the limiting theoretical overlap $\bar q$ - and thus as $\lambda$ decreases higher values of $n$ would thus be needed to match closely $\bar q$.

Finally, in the regime $\lambda <1$, we observe on Figure \ref{fig:experimental_runs4} that similarity measure 
$p_1$ explodes and overtakes the risk.

\begin{figure}
    \centering
    \subfigure[Overlap]{\includegraphics[width=6cm]{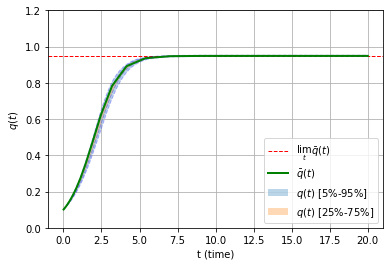} }
    \qquad
    \subfigure[Cost]{\includegraphics[width=6cm]{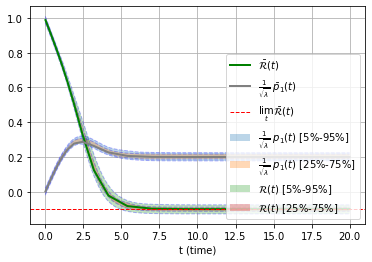} }
    \caption{$\lambda = 10, n = 1000, \alpha = 0.1, \delta_t=0.1$}
    \label{fig:experimental_runs2}
\end{figure}
\begin{figure}
    \centering
    \subfigure[Overlap]{\includegraphics[width=6cm]{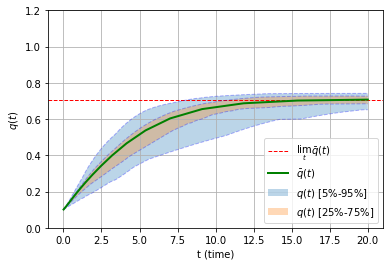} }
    \qquad
    \subfigure[Cost]{\includegraphics[width=6cm]{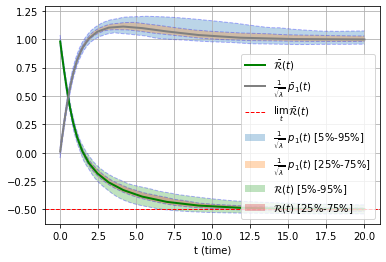} }
    \caption{$\lambda = 2, n = 1000, \alpha = 0.1, \delta_t=0.1$}
    \label{fig:experimental_runs3}
\end{figure}
\begin{figure}
    \centering
    \subfigure[Overlap]{\includegraphics[width=6cm]{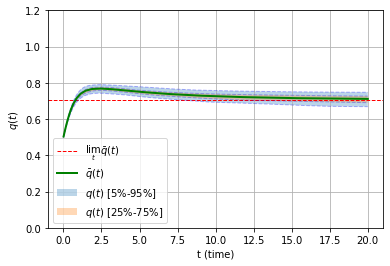} }
    \qquad
    \subfigure[Cost]{\includegraphics[width=6cm]{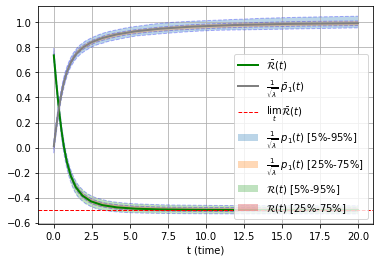} }
    \caption{$\lambda = 2, n = 1000, \alpha = 0.5, \delta_t=0.1$}
    \label{fig:experimental_runs5}
\end{figure}
\begin{figure}
    \centering
    \subfigure[Overlap]{\includegraphics[width=6cm]{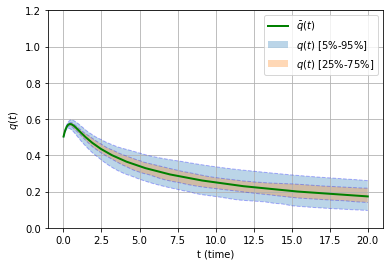} }
    \qquad
    \subfigure[Cost]{\includegraphics[width=6cm]{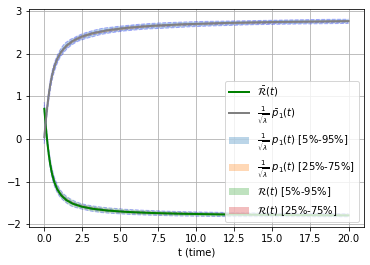} }
    \caption{$\lambda = 0.5, n = 1000, \alpha = 0.5, \delta_t=0.1$. Note the different scale for the cost.}
    \label{fig:experimental_runs4}
\end{figure}

\end{document}